\RequirePackage{fix-cm}
\documentclass{article} % For LaTeX2e
\usepackage{iclr2027_conference,times}

\usepackage{amsmath,amsfonts,bm}

\def\eqref#1{equation~\ref{#1}}
\def\1{\bm{1}}

\def\vzero{{\bm{0}}}

\def\ve{{\bm{e}}}

\def\vq{{\bm{q}}}

\def\vv{{\bm{v}}}

\def\mI{{\bm{I}}}

\def\mP{{\bm{P}}}

\def\mZ{{\bm{Z}}}

\DeclareMathAlphabet{\mathsfit}{\encodingdefault}{\sfdefault}{m}{sl}
\SetMathAlphabet{\mathsfit}{bold}{\encodingdefault}{\sfdefault}{bx}{n}

\def\gG{{\mathcal{G}}}

\def\gV{{\mathcal{V}}}

\usepackage{microtype}
\usepackage{hyperref}
\usepackage{url}
\usepackage{booktabs}

\newtheorem{theorem}{Theorem}
\newtheorem{lemma}{Lemma}
\newtheorem{corollary}{Corollary}

\newenvironment{proof}[1][Proof]
{\par\noindent\textbf{#1.}\ }
{\hfill$\square$\par}
\usepackage{algorithm}
\usepackage{algpseudocode}
\usepackage{float}
\usepackage{caption}
\usepackage{graphicx}
\usepackage{amssymb}
\usepackage{wrapfig}
\usepackage{booktabs}
\usepackage[table]{xcolor}
\usepackage{lineno}
\usepackage{multirow}
\definecolor{mazebest}{HTML}{E3F1E5}
\definecolor{mazemax}{HTML}{FFF2CB}

\usepackage{xcolor}

\definecolor{darkblue}{rgb}{0, 0, 0.5}
\hypersetup{colorlinks=true, citecolor=darkblue, linkcolor=darkblue, urlcolor=darkblue}
\usepackage{xurl}
\DeclareRobustCommand{\capmark}[2]{%
  \begingroup
  \setlength{\fboxsep}{0pt}%
  \smash{\colorbox{#1}{#2}}\vphantom{#2}%
  \endgroup
}

\title{Efficient Reasoning with Flow Language\\Models}

\author{%
    Hanru Bai\textsuperscript{\textnormal{1,2}},
    Faissal Izermine\textsuperscript{\textnormal{1,3}},
    Oscar Davis\textsuperscript{\textnormal{4}},
    T. Konstantin Rusch\textsuperscript{\textnormal{1,3,5,6}\thanks{\scriptsize Correspondence to \href{mailto:tkrusch@tue.ellis.eu}{\texttt{tkrusch@tue.ellis.eu}}}~}\\
    {\small \textsuperscript{1}Max Planck Institute for Intelligent Systems,}
    {\small \textsuperscript{2}ETH Zurich,}
    {\small \textsuperscript{3}ELLIS Institute Tübingen,} \\
    {\small \textsuperscript{4}University of Oxford,} 
    {\small \textsuperscript{5}Tübingen AI Center,}
    {\small \textsuperscript{6}Liquid AI}
}

\usepackage{tcolorbox}
\tcbuselibrary{breakable}

\usepackage{iclrtheoremhighlight}
\iclrfinalcopy % Uncomment for camera-ready version, but NOT for submission.
\begin{document}

\maketitle

% \begin{abstract}
% We study the reasoning efficiency of continuous Flow Language Models (FLMs)
% relative to discrete diffusion language models. Unlike discrete diffusion,
% which progressively commits positions to hard tokens, FLM maintains a
% continuous state and discretizes only after denoising. We theoretically show
% that this continuous state can represent geometric superpositions of candidate sequences and accumulate their soft evidence along the FLM trajectories. Experiments on graph planning and Sudoku show that, under small denoising-evaluation budgets, FLM achieves higher reasoning accuracy than matched discrete diffusion baselines. Intermediate-state interventions further provide evidence that information beyond the top candidate contributes to subsequent
% refinement. These results suggest that continuous states provide a useful
% representation for efficient reasoning with limited denoising evaluations.
% \end{abstract}

\begin{abstract}
% Diffusion language models can operate in discrete or continuous state spaces; yet, how this choice shapes their reasoning capabilities remains poorly understood.
Flow Language Models (FLMs) have emerged as a continuous-state alternative to discrete diffusion language models, yet the role of their continuous representations in reasoning remains unclear.
We investigate this question by comparing the reasoning efficiency of FLMs and discrete diffusion models, measured by solution accuracy under matched denoising steps. Unlike discrete diffusion, which passes categorical states between denoising steps, FLMs evolve a continuous sequence representation throughout denoising and decodes it into discrete tokens only at the end.
Our theoretical analysis shows, from a superposition perspective, how information retained in these continuous states can benefit reasoning. Intermediate-state interventions provide further empirical support for this theoretical account, showing that removing information about alternative candidates reduces subsequent solution recovery. Together, these findings show that FLMs allow evidence for multiple candidates to persist and inform subsequent reasoning before a discrete answer is produced. 
Furthermore, our experiments on maze planning and Sudoku tasks show that FLMs achieve greater reasoning efficiency in the few-step regime: FLMs achieves higher sequence accuracy than discrete diffusion baselines at matched model sizes and small denoising steps. 
On maze planning tasks, FLMs can also achieve comparable accuracy with smaller models. For example, on Maze15, FLM reaches the 95\% accuracy target at 64 denoising steps with 36.5\% fewer parameters than MDLM. These findings point to continuous state spaces as a promising foundation for reasoning models that require fewer refinement steps.

\end{abstract}

% \TODO{Maybe here "but how the choice between continuous and discrete state spaces affects reasoning remains less understood. " the best is to start by saying that choosing "continuous over discrete remains unclear" since people were working a lot on discrete diffusion to make it work for reasoning and thus we are not the first to do it?}

\section{Introduction}
\label{sec:introduction}

Reasoning is essential for many real-world decision-making problems. It requires models to integrate partial information and refine intermediate computational
states before reaching a valid solution
\citep{xu2025toward,patil2025advancing}. The reasoning capabilities of autoregressive models have been extensively studied, and their capabilities largely extended through chain-of-thought
prompting \citep{wei2022chain} and more recent continuous chain-of-thought variants
\citep{hao2024training,zhu2026reasoning}.
% These methods spend additional computational steps on intermediate thoughts, but also organize the reasoning process as a sequential chain. Prior work has shown that such chain-style reasoning can be limiting for tasks requiring exploration, planning, or multiple candidate solution paths \citep{yao2023tree}. This motivates growing interest in alternative generation paradigms for reasoning \citep{ye2025beyond}. 
These methods allocate additional computation to a sequential chain of intermediate thoughts. While effective on many tasks, this sequential formulation alone may lack the flexibility needed for exploration, planning, or comparing multiple candidate solution paths \citep{yao2023tree}. This motivates the study of alternative generation paradigms for reasoning.

% Among these alternatives, diffusion language models \citep{austin2021structured,lou2023discrete,sahoo2024simple} provide a alternative modeling paradigm for reasoning. Rather than committing to a left-to-right sequence of intermediate decisions, they refine a sequence-level state over multiple denoising steps. At each step, bidirectional attention allows every position to interact with the current global state, enabling information to be revised and propagated across the sequence \citep{chen2026bifocal}. The number of denoising steps therefore is an additional computational dimension along which partial constraints and candidate solutions can be progressively coordinated \citep{zhu2025surveylatentreasoning,chen2025reasoning}. A central question, however, is what information these intermediate states preserve, which depends critically on the state space in which denoising is performed.
Among these alternatives, diffusion language models provide a framework for iterative sequence generation \citep{austin2021structured,lou2023discrete,sahoo2024simple}. They refine a sequence over multiple denoising steps without requiring a fixed left-to-right generation order. The denoising process allows multiple token positions to be updated together, with some formulations also permitting previously generated tokens to be revised in subsequent steps. This ability to jointly predict undecoded tokens can also benefit reasoning tasks that require planning and global coherence \citep{he2026reasoning}.

% Since language is typically represented by discrete tokens, most diffusion language models operate directly on a discrete state space. Representative models such as D3PM \citep{austin2021structured}, SEDD \citep{lou2023discrete}, MDLM \citep{sahoo2024simple}, LLaDA \citep{nie2026large}, and Dream \citep{ye2025dream} have made discrete diffusion increasingly competitive through improved training objectives and sampling strategies. This imposes an abrupt loss of information at every integration step: the denoising model produces an output distribution, from which is only retained a single discrete token per position. Relative evidence among non-selected candidates is therefore not explicitly carried into subsequent computation. Prematurely instantiated tokens may also influence updates at other positions through bidirectional attention, making refinement sensitive to early local decisions. 
% It therefore remains unclear whether discrete states provide a sufficiently expressive medium for iterative reasoning.
Since language is typically represented by discrete tokens, many diffusion language models perform denoising directly in a discrete state space \citep{austin2021structured,lou2023discrete,sahoo2024simple,ye2025dream}. In standard discrete sampling, the denoiser uses interactions across positions to predict token probabilities, but the state passed to the next step contains only a single categorical value per position, which may be a token or a mask. The probabilities assigned to alternative tokens are therefore not explicitly preserved in this state. Early token choices may also influence subsequent predictions at other positions through bidirectional attention, making refinement sensitive to local decisions made before sufficient information has been integrated. It therefore remains unclear whether discrete states provide a sufficiently expressive medium for iterative reasoning.

% Motivated by latent reasoning \citep{hao2024training}, we ask whether reasoning can instead unfold in an abstract continuous space without committing to a single categorical state at each denoising step. Flow Language Models (FLMs) \citep{lee2026flow,roosCategoricalFlowMaps2026,potaptchik2026discreteflowmaps} instantiate this idea by evolving a continuous sequence representation throughout denoising and decoding discrete tokens only at the end. Their intermediate states can retain soft evidence over multiple candidate tokens and partial solutions, allowing subsequent computation to refine them before commitment.

% Although recent methods such as ELF \citep{hu2026elf}, LangFlow \citep{chen2026langflow}, and Plaid \citep{plaid,yang2026continuous} have demonstrated the promise of continuous diffusion and flow matching for language modeling, their reasoning efficiency remains largely unexplored. Our goal is not to evaluate general language modeling ability. Instead, we isolate reasoning capability on structured, verifiable tasks and measure reasoning efficiency by the solution accuracy achieved under a fixed number of denoisng steps. Under matched inference budgets, we compare continuous FLMs with discrete diffusion models and show, both theoretically and empirically, that continuous states provide a better accuracy-compute trade-off, particularly when only few denoising steps are used. Our main contributions are the following.
Motivated by latent reasoning \citep{hao2024training}, we ask whether reasoning can instead unfold in a continuous space without committing to a categorical state at each denoising step. Flow Language Models (FLMs) \citep{lee2026flow,roosCategoricalFlowMaps2026,potaptchik2026discreteflowmaps} realize this idea by evolving continuous sequence representations throughout denoising and decoding discrete tokens only at the end. Their intermediate states can retain soft representations of multiple candidate tokens, allowing subsequent computation to refine these alternatives before producing a discrete answer.

Recent work has developed continuous diffusion and flow-matching methods for language modeling \citep{hu2026elf,chen2026langflow}, but their reasoning efficiency remains less understood. Our goal is to compare the reasoning efficiency of continuous FLMs with discrete diffusion models, measured by solution accuracy under a fixed budget of denoiser evaluations. To isolate reasoning capability, we focus on structured, verifiable tasks. Our results indicate that continuous FLMs can offer a better accuracy--computation trade-off than discrete diffusion models in the settings studied, particularly in the few-step regime. Our main contributions are the following.

%\paragraph{Contributions.}
\begin{itemize}
\item We systematically study the reasoning efficiency of continuous FLMs relative to discrete diffusion models, measured by accuracy under a fixed denoising step. Beyond final accuracy alone, we examine the accuracy -- computation trade-off throughout iterative refinement, emphasizing the few-step regime.

\item From a superposition perspective, we identify a theoretical mechanism by which continuous states preserve useful intermediate information and enable its use in subsequent denoising steps. Candidate-retention measurements, controlled interventions, and decoder probes provide empirical support for this mechanism.

\item Experimentally, we demonstrate that FLMs achieve stronger few-step reasoning efficiency than discrete diffusion baselines across both maze planning and Sudoku, attaining higher sequence accuracy under matched model sizes and small denoising steps. On the maze benchmarks specifically, the advantage of FLMs over MDLM becomes more pronounced as the task difficulty increases, and FLMs reach matched accuracy targets with smaller models.

\end{itemize}

\section{Problem Setup}
\label{sec:problem-setup}

% We use graph planning as a representative task to introduce our reasoning
% formulation, as it requires combining local relational information into a
% globally valid solution \citep{chai2025graphllm}. 
% In this section, we first
% describe the tokenized representation and then the conditional flow-matching model.
% The framework is not restricted to graph planning and is also evaluated on
% Sudoku.

% \subsection{Graph Planning}
% \label{subsec:graph-planning}

% Let $\gG=(\gV,\gE)$ be a graph and let $\vq=(r,\tau)$ specify a source
% $r\in\gV$ and a target $\tau\in\gV$. Given $(\gG,\vq)$, the task is to predict
% a path
% $
% Y=(y_0,\ldots,y_L)
% \label{eq:path}
% $
% such thattg05t-`j
% \begin{equation}
% y_0=r,\qquad y_L=\tau,\qquad
% (y_i,y_{i+1})\in\gE
% \quad\text{for }i=0,\ldots,L-1.
% \label{eq:valid-path}
% \end{equation}

\paragraph{Tokenized sequence representation.}
\label{subsec:tokenized-representation}
% The graph and query are serialized into a conditioning sequence
% $\mathbf c=(c_1,\ldots,c_{M_c})$, while the target path is serialized into an
% answer sequence $\mathbf x=(x_1,\ldots,x_{M_x})$. Let
% $\mathcal V_{\mathrm{tok}}$ denote the tokenizer vocabulary,
% $K=|\mathcal V_{\mathrm{tok}}|$, and $\ve_a\in\mathbb R^K$ denote the one-hot
% vector for token $a$. For a sequence $\mathbf s=(s_1,\ldots,s_m)$, define
% $
% \operatorname{OH}(\mathbf s)
% =
% [\ve_{s_1},\ldots,\ve_{s_m}]^\top
% \in\mathbb R^{m\times K}.
% $
% The conditioning and clean answer states are
% \begin{equation}
% \mathbf C_{\gG,\vq}
% =
% \operatorname{OH}(\mathbf c)
% \in\mathbb R^{M_c\times K},
% \qquad
% \mZ_1
% =
% \operatorname{OH}(\mathbf x)
% \in\mathbb R^{M_x\times K}.
% \label{eq:tokenized-sequence-states}
% \end{equation}
Let $\mathcal V_{\mathrm{tok}}$ denote the tokenizer vocabulary and
$K=|\mathcal V_{\mathrm{tok}}|$. Each conditioning input and target output are
serialized and tokenized into sequences
$\mathbf c=(c_1,\ldots,c_{M_c})$ and
$\mathbf x=(x_1,\ldots,x_{M_x})$    . Let
$\ve_a\in\mathbb R^K$ denote the one-hot vector associated with token
$a\in\mathcal V_{\mathrm{tok}}$. For a token sequence
$\mathbf s=(s_1,\ldots,s_m)$, we write the one-hot representation of $\mathbf s$ as
$
\operatorname{OH}(\mathbf s)
=
[\ve_{s_1},\ldots,\ve_{s_m}]^\top
\in\mathbb R^{m\times K}.
$
The conditioning and clean target states are
\begin{equation}
\textstyle
\mathbf C
=
\operatorname{OH}(\mathbf c)
\in\mathbb R^{M_c\times K},
\qquad
\mZ_1
=
\operatorname{OH}(\mathbf x)
\in\mathbb R^{M_x\times K}.
\label{eq:tokenized-sequence-states}
\end{equation}

\paragraph{Conditional flow matching.}
\label{subsec:conditional-flow-matching}
FLM keeps $\mathbf C_{\gG,\vq}$ fixed and evolves a continuous answer state
$\mZ_t\in\mathbb R^{M_x\times K}$. During training, we sample
$\mZ_0\sim\mathcal N(\vzero,\mI)$ and construct the linear interpolation
\begin{equation}
\textstyle
\mZ_t=(1-t)\mZ_0+t\mZ_1,
\qquad
t\sim\mathcal{U}[0,1].
\label{eq:fm-interpolation}
\end{equation}
Given $(\mZ_t,t,\mathbf C_{\gG,\vq})$, the model predicts a distribution over
the tokenizer vocabulary at each answer position:
$
\mP_\theta(\mZ_t,t,\mathbf C_{\gG,\vq})
\in(\Delta^{K-1})^{M_x}.
\label{eq:clean-token-distribution}
$
We train the model to recover clean answer tokens using
\begin{equation}
\textstyle
\mathcal L_{\mathrm{FM}}(\theta)
=
\mathbb E
\left[
-\sum_{i=1}^{M_x}
\log
\mP_\theta(\mZ_t,t,\mathbf C_{\gG,\vq})_{i,x_i}
\right].
\label{eq:fm-token-ce-objective}
\end{equation}

At inference time, we initialize $\mZ_0\sim\mathcal N(\vzero,\mI)$ and
numerically integrate
\begin{equation}
\textstyle
\frac{d\mZ_t}{dt}
=
\frac{
\mP_\theta(\mZ_t,t,\mathbf C_{\gG,\vq})-\mZ_t
}{1-t},
\qquad 0\leq t<1.
\label{eq:conditional-token-flow}
\end{equation}
Using $N$ integration steps requires $N$ forward passes. After the
final update, the answer is discretized once as
$
\widehat x_i
=
\operatorname*{arg\,max}_{a\in\mathcal V_{\mathrm{tok}}}
(\mZ_1)_{i,a}$,
for all $
 i=1,\ldots,M_x.
\label{eq:decode-answer-token}
$
The decoded token sequence is then detokenized and parsed into the predicted
path, $\widehat Y$.

\section{Theoretical Analysis}
\label{sec:theory}

% Our analysis provides a representational perspective on why FLMs can achieve higher reasoning accuracy under small denoising-evaluation budgets than discrete diffusion models. In particular, we characterize the ability of a persistent continuous state to retain information associated with multiple
% candidate sequences. 
We begin by developing theoretically a representational account of how continuous FLM states preserve and use information not captured by their hard-token projections. Section~\ref{sec:theory-candidate-superposition} shows that FLM denoising accumulates a temporal superposition of candidate sequences. Section~\ref{sec:theory-historical-influence} characterizes how retained non-top-1 historical support can affect model computation and clean-token margins without changing the current hard state. Finally, Section~\ref{sec:theory-task-information} quantifies the task-relevant predictive information available beyond hard projection through decoder-risk bounds.

\subsection{Temporal Candidate-Sequence Superposition}
\label{sec:theory-candidate-superposition}
Let $\mathcal S_{M_x}=\mathcal V_{\mathrm{tok}}^{M_x}$ denote the set of candidate token sequences. Each $\mathbf s=(s_1,\ldots,s_{M_x})\in\mathcal S_{M_x}$ corresponds to the one-hot vertex
$
\mathbf E^{(\mathbf s)}
=\operatorname{OH}(\mathbf s)$,
$
\mathcal H = \left\{
\mathbf E^{(\mathbf s)}:
\mathbf s\in\mathcal S_{M_x}
\right\},
$
where $\mathcal H$ is the set of complete hard-token states. At time $t$, let
$
\mP_t
=
\mP_\theta(\mZ_t,t,\mathbf C_{\gG,\vq})
$
be the position-wise clean-token prediction and define
$
\pi_{\mathbf s}(t)
=
\prod_{j=1}^{M_x}(\mP_t)_{j,s_j}.
\label{eq:instantaneous-candidate-weight}
$
These weights form a probability distribution and yield the exact
decomposition
$
\mP_t
=
\sum_{\mathbf s\in\mathcal S_{M_x}}
\pi_{\mathbf s}(t)\mathbf E^{(\mathbf s)}.
\label{eq:instantaneous-candidate-decomposition}
$
Thus, each instantaneous clean prediction is a convex superposition of
candidate-sequence vertices.

\begin{highlight}
\begin{theorem}[Candidate-sequence superposition in FLM]
\label{thm:flm-candidate-superposition}
Suppose that the FLM state follows
\begin{equation}
\textstyle
\frac{d\mZ_t}{dt}
=
\frac{\mP_t-\mZ_t}{1-t},
\qquad 0\leq t<1,
\label{eq:superposition-flm-ode}
\end{equation}
and that $\mP_t$ is integrable along the trajectory. Then, for every
$0<t<1$,
\begin{equation}
\textstyle
\mZ_t
=
(1-t)\mZ_0+t\,\overline{\mathbf E}_t,
\qquad
\overline{\mathbf E}_t
=
\sum_{\mathbf s\in\mathcal S_{M_x}}
\overline{\pi}_{\mathbf s}(t)\mathbf E^{(\mathbf s)},
\label{eq:normalized-candidate-superposition}
\end{equation}
where
\begin{equation}
\textstyle
\overline{\pi}_{\mathbf s}(t)
=
\frac{1-t}{t}
\int_0^t
\frac{\pi_{\mathbf s}(u)}{(1-u)^2}\,du,
\qquad
\overline{\pi}_{\mathbf s}(t)\geq0,
\qquad
\sum_{\mathbf s}\overline{\pi}_{\mathbf s}(t)=1.
\label{eq:integrated-candidate-weight}
\end{equation}
Therefore, apart from the residual initial noise $(1-t)\mZ_0$, the FLM state
contains a temporally accumulated convex superposition of candidate-sequence
vertices. If two distinct sequences have positive integrated weights, then
$\overline{\mathbf E}_t\notin\mathcal H$; hence, this denoised component cannot
be represented by any single complete hard-token state. \highlightend
\end{theorem}
\end{highlight}
Theorem~\ref{thm:flm-candidate-superposition} shows that the FLM state preserves historical support for non-top-1 candidates through a discounted superposition of past token distributions. Our intermediate-state intervention experiments further suggest that FLM
uses those superposed candidate information for subsequent refinement during
reasoning (Section~\ref{sec:superposition-analysis}). 
We next characterize how historical candidate information retained by superposition can affect clean-token predictions
without changing the current argmax. 
% The subsequent corollary tracks these effects through the Euler update. 
Discrete solvers retaining only the current hard-token sequence
cannot distinguish such histories at fixed time.

\subsection{Historical Influence Within a Fixed Hard State}
\label{sec:theory-historical-influence}
We compare the original state with one in which a non-top-1
candidate's historical contribution is removed.
Fix the context, time \(t\in(0,1)\), parameters, and attention mask.
Let \(\Pi\) denote rowwise argmax with fixed tie-breaking, and
\(\mathbf x=\Pi(\mZ_t)\) the current hard-token sequence.
For a strict non-top-1 candidate \(v\) at free position \(j\), define
\[
\textstyle
h_t=t(\overline{\mathbf E}_t)_{j,v}>0,\qquad
\mathbf U_t(\lambda)=\mZ_t-(h_t-\lambda)\mathbf D^{jv},
\qquad
\mathbf D^{jv}=\mathbf e_j\mathbf e_v^\top,
\]
where \(\lambda\in[0,h_t]\) is the retained historical support.
Let \(\mP_t^{(\lambda)}\) be the full prediction recomputed on
\(\mathbf U_t(\lambda)\), and write
\(M_t^{(\lambda)}
=(\mP_t^{(\lambda)})_{j,v}-(\mP_t^{(\lambda)})_{j,x_j}\),
with \(M_t=M_t^{(h_t)}\).

For first-layer scaled dot-product attention, denote the queries,
keys, values, and softmax weights by
\(\mathbf q_i^b,\mathbf k_m^b,\mathbf v_m^b,a_{im}^b\), respectively,
with \(\mathbf o_i^b=\sum_m a_{im}^b\mathbf v_m^b\).
Write \(M_t^{(\lambda)}=F_t(O(\lambda),\xi(\lambda))\):
\(O\) collects all head outputs at all positions, \(\xi\) collects
bypass inputs, and \(F_t\) includes all subsequent computation.
With overdots denoting derivatives with respect to \(\lambda\), set
\[
\begin{aligned}
\textstyle
\mathbf g_i^b=\nabla_{\mathbf o_i^b}F_t, \quad
 \eta_{im}^b=\langle\mathbf g_i^b,\mathbf v_m^b\rangle, \quad
\rho_{im}^{Q,b} 
=\frac{\langle\dot{\mathbf q}_i^b,\mathbf k_m^b\rangle}{\sqrt{d_b}}, \quad
\rho_{im}^{K,b} =\frac{\langle\mathbf q_i^b,\dot{\mathbf k}_m^b\rangle}{\sqrt{d_b}},
\end{aligned}
\]
where \(d_b\) is the head dimension.
All quantities depend on \(\lambda\); expectations and covariances
below use the attention distribution \(m\sim a_i^b\) over visible
positions.

\begin{highlight}
\begin{theorem}[Historical influence within a fixed hard state]
\label{thm:history-mechanism-full-influence}
\label{thm:full-history-attention}
\label{thm:aligned-history-attention}
% Assume deterministic evaluation and \(C^1\) model maps near the
% intervention path. 
Assume deterministic evaluation and that the query, key, value,
and bypass maps, together with the downstream map \(F_t\), are
continuously differentiable on open sets containing their
respective inputs along the intervention path.
For every \(\lambda\in[0,h_t]\),
\begin{equation}
\textstyle
\Pi(\mathbf U_t(\lambda))=\mathbf x,\;
\partial_\lambda M_t^{(\lambda)}
=\mathcal C_Q+\mathcal C_K+\mathcal C_V+\mathcal C_R
=:\mathcal C(\lambda), \text{where}
\label{eq:full-history-hard-preservation}
\end{equation}
\begin{equation}
\begin{aligned}
\textstyle
\mathcal C_X
=\sum_{b,i}\operatorname{Cov}_{m\sim a_i^b}
  (\rho_{im}^{X,b},\eta_{im}^b),
  \; X\in\{Q,K\}; \;
\mathcal C_V
=\sum_{b,i}\mathbb E_{m\sim a_i^b}
  \langle\mathbf g_i^b,\dot{\mathbf v}_m^b\rangle,
\;
\mathcal C_R=\langle\nabla_\xi F_t,\dot\xi\rangle.
\end{aligned}
\label{eq:full-history-path-contributions}
\end{equation}
If the maps preceding attention act independently at each position,
then for \(i\neq j\) with \(j\) visible,
\begin{equation}
\textstyle
\dot{\mathbf o}_i^b
=a_{ij}^b(\mathbf v_j^b-\mathbf o_i^b)\rho_{ij}^{K,b}
+a_{ij}^b\dot{\mathbf v}_j^b.
\label{eq:history-cross-position}
\end{equation}
Wherever \(\mathcal C(\lambda)\neq0\), sufficiently small admissible
nonzero perturbations change the margin without changing \(\mathbf x\).
Moreover,
\begin{equation}
\textstyle
\Delta_t:=M_t-M_t^{(0)}
=\int_0^{h_t}\mathcal C(\lambda)\,d\lambda.
\label{eq:full-history-integrated-effect}
\end{equation}

Let \(\Omega_t=\max_\lambda M_t^{(\lambda)}
-\min_\lambda M_t^{(\lambda)}\) over this path.  \(\Omega_t>0\) if and only if
\(\mathcal C\not\equiv0\). Every scalar readout \(H\) using only context, time, and the hard
state satisfies
\begin{equation}
\textstyle
\sup_{\lambda\in[0,h_t]}
\left|M_t^{(\lambda)}
-H(\mathbf C_{\gG,\vq},t,\Pi(\mathbf U_t(\lambda)))\right|
\geq\frac{\Omega_t}{2}\geq\frac{|\Delta_t|}{2}.
\label{eq:history-mechanism-hard-error}
\end{equation} 
The bound is tight.
\highlightend
% The first bound is tight, and \(\Omega_t>0\) if and only if
% \(\mathcal C\not\equiv0\).
\end{theorem}
\end{highlight}

Equation~\ref{eq:full-history-path-contributions} separates attention
reweighting through Q/K, message changes through V, and bypass
responses; their signed effects may reinforce or cancel.
Equation~\ref{eq:history-cross-position} makes the influence on other
positions explicit, while Equation~\ref{eq:history-mechanism-hard-error}
bounds the margin variation that a readout of the hard state cannot reproduce.   We experimentally test this mechanism by removing retained non-top-1 historical support while preserving the hard-token sequence and measuring the resulting
prediction-margin changes and Q/K/V/bypass contribution in Appendix~\ref{app:full-state-intervention}.

Appendix~\ref{app:history-refinement} further shows how these differences propagate through Euler updates and when the original trajectory strengthens a candidate's relative support without changing the hard-token sequence. We next examine whether these differences help predict the correct solution. The following theorem quantifies this distinction through bounds on prediction risk.

\subsection{Task-Relevant Information Beyond Hard Projection}
\label{sec:theory-task-information}
At fixed time $t$, we predict target-path membership
$Y=\mathbf 1\{X_J=v_{\mathrm{path}}\}$, where $X_J$ is the target
token at a uniformly sampled output position $J$.
Inputs $A_t,B_t$ share the context in a window $\mathcal N$ around
$J$, using respectively the continuous state and its rowwise one-hot
argmax with fixed tie-breaking.
Let $R_t^Z,R_t^{\mathrm{hard}}$ denote their unweighted population
BCE risks, $\Delta_t^{\mathrm{dec}}=R_t^{\mathrm{hard}}-R_t^Z$,
and $\varepsilon_t^Z,\varepsilon_t^{\mathrm{hard}}$ their excess
risks above the corresponding Bayes risks.
For $\mathbf W_t=\operatorname{softmax}((\mZ_t)_{\mathcal N})$,
define $A_t^{\mathrm{soft}}=(\mathbf C_{\mathcal N},\mathbf W_t)$
and $A_t^{\mathrm{conf}}=(B_t,\max\mathbf W_t)$,
with softmax and maxima taken rowwise.

\begin{highlight}
\begin{theorem}[Task-relevant information beyond hard projection]
\label{thm:flm-decoder-accessibility}
Assume finite decoder risks and let
\(b_t\geq
(\mathbb E\|(\mZ_t)_{J,:}-\mathbf e_{X_J}^{\top}\|_2^2)^{1/2}\).
Define the predictive signal beyond winner confidence by
\begin{equation}
\textstyle
\nu_t=\mathbb E\left[
\operatorname{Var}\left(
\Pr(Y=1\mid A_t^{\mathrm{soft}})
\mid A_t^{\mathrm{conf}}
\right)\right].
\label{eq:decoder-beyond-confidence}
\end{equation}
Then, with \(h_{\mathrm{bin}}\) denoting binary entropy in nats,
\begin{equation}
\textstyle
2\nu_t-\varepsilon_t^Z
\leq\Delta_t^{\mathrm{dec}}
\leq h_{\mathrm{bin}}\left(\min\{1/2,2b_t^2\}\right)
+\varepsilon_t^{\mathrm{hard}}.
\label{eq:decoder-main-bound}
\end{equation}
Also,
\(R_t^Z\leq
h_{\mathrm{bin}}(\min\{1/2,2b_t^2\})+\varepsilon_t^Z\). \highlightend
\end{theorem}
\end{highlight}

The left-hand inequality in Equation~\ref{eq:decoder-main-bound} guarantees a predictive advantage when \(2\nu_t>\varepsilon_t^Z\); the right-hand inequality implies that the Bayes-risk gap vanishes as \(b_t\to0\).
We examine this retained information experimentally by using matched decoders on intermediate FLM states and their hard projections, measuring how well each representation supports prediction of the correct solution in Section~\ref{sec:superposition-analysis}.
All detailed proofs are provided in the Appendix~\ref{app:theorem proof}.

\begin{table}[t]
\centering
\caption{
\textbf{Sequence accuracy (\%) of FLM and baseline methods on Maze10,
Maze15, and Maze20 across denoising steps.}
+ Ent. and + Conf. denote entropy-based and confidence-based remasking,
respectively.
For each task and method, \capmark{mazebest}{green} marks the first
tested step reaching at least $95\%$;
\capmark{mazemax}{yellow} marks the first attaining its observed maximum.
}
\label{tab:maze-main}
\begingroup
\definecolor{mazebest}{HTML}{E3F1E5}
\definecolor{mazemax}{HTML}{FFF2CB}
\fontsize{9}{10.5}\selectfont
\setlength{\tabcolsep}{0pt}
\renewcommand{\arraystretch}{1.06}
\def\mazedual#1{\cellcolor{mazemax}%
  \makebox[\linewidth][l]{%
    \rlap{\color{mazebest}%
      \rule[-\arraystretch\dp\strutbox]{.5\linewidth}%
        {\dimexpr\arraystretch\ht\strutbox+\arraystretch\dp\strutbox\relax}}%
    \makebox[\linewidth][c]{\strut\textbf{#1}}}}
\begin{tabular}{@{}>{\raggedright\arraybackslash}m{\dimexpr0.088607595\textwidth-0.664556962pt\relax}>{\centering\arraybackslash}m{\dimexpr0.078481013\textwidth-0.588607595pt\relax}>{\centering\arraybackslash}m{\dimexpr0.058227848\textwidth-0.436708861pt\relax}@{\hspace{2.5pt}}>{\centering\arraybackslash}m{\dimexpr0.068354430\textwidth-0.512658228pt\relax}@{\hspace{2.5pt}}>{\centering\arraybackslash}m{\dimexpr0.073417722\textwidth-0.550632911pt\relax}>{\centering\arraybackslash}m{\dimexpr0.068354430\textwidth-0.512658228pt\relax}>{\centering\arraybackslash}m{\dimexpr0.081012658\textwidth-0.607594937pt\relax}>{\centering\arraybackslash}m{\dimexpr0.083544304\textwidth-0.626582278pt\relax}>{\centering\arraybackslash}m{\dimexpr0.078481013\textwidth-0.588607595pt\relax}>{\centering\arraybackslash}m{\dimexpr0.086075949\textwidth-0.645569620pt\relax}>{\centering\arraybackslash}m{\dimexpr0.086075949\textwidth-0.645569620pt\relax}>{\centering\arraybackslash}m{\dimexpr0.086075949\textwidth-0.645569620pt\relax}@{\hspace{2.5pt}}>{\centering\arraybackslash}m{\dimexpr0.063291139\textwidth-0.474683544pt\relax}@{}}
\toprule
{\fontsize{8.5}{9.8}\selectfont Task} & {\fontsize{8.5}{9.8}\selectfont \shortstack{\textit{Params}\\\textit{(M)}}} & {\fontsize{8.5}{9.8}\selectfont $N$} & {\fontsize{8.5}{9.8}\selectfont CoT} & {\fontsize{8.5}{9.8}\selectfont D3PM} & {\fontsize{8.5}{9.8}\selectfont SEDD} & {\fontsize{8.5}{9.8}\selectfont \shortstack{SEDD\\+ Ent.}} & {\fontsize{8.5}{9.8}\selectfont \shortstack{SEDD\\+ Conf.}} & {\fontsize{8.5}{9.8}\selectfont MDLM} & {\fontsize{8.5}{9.8}\selectfont ReMDM} & {\fontsize{8.5}{9.8}\selectfont \shortstack{MDLM\\+ Ent.}} & {\fontsize{8.5}{9.8}\selectfont \shortstack{MDLM\\+ Conf.}} & {\fontsize{8.5}{9.8}\selectfont \textbf{FLM}} \\
\cmidrule(r{2pt}){1-3}\cmidrule(lr{2pt}){4-4}\cmidrule(lr{2pt}){5-12}\cmidrule(l{2pt}){13-13}
\multirow{10}{*}{\textbf{Maze10}}
& \multirow[c]{10}{=}{\centering \textit{59.11}}
& 1
& \multirow[c]{10}{=}{\centering 50.5}
& 77.5 & 94.0 & 94.0 & 94.0
& 87.5 & 87.5
& \cellcolor{mazemax}\textbf{87.2}
& 87.5
& \cellcolor{mazebest}\textbf{95.7} \\

& & 2 & & 82.8
& \cellcolor{mazebest}\textbf{96.3}
& \mazedual{99.7}
& \cellcolor{mazebest}\textbf{98.8}
& 89.0 & 89.0 & 66.3 & 93.1 & 99.5 \\

& & 4 & & 85.9
& 96.7 & 98.3
& \textcolor{black!48}{\textit{NR}}
& 92.2 & 92.3 & 66.2
& \cellcolor{mazebest}\textbf{95.4}
& 99.5 \\

& & 8 & & 91.4
& 98.8
& 97.0 & 99.4
& 94.0 & 94.2 & 71.4 & 96.2 & 99.5 \\

& & 16 & & 94.2
& 98.9 & 96.9
& \cellcolor{mazemax}\textbf{99.7}
& \cellcolor{mazebest}\textbf{96.3}
& \cellcolor{mazebest}\textbf{96.4}
& 69.2 & 97.7 & 99.5 \\

& & 32 & & \cellcolor{mazebest}\textbf{95.7}
& 99.3 & 96.7 & 99.5
& 97.8 & 97.9 & 75.9 & 98.5 & 99.5 \\

& & 64 & & 96.8
& 99.7 & 96.5 & 99.6
& 98.0 & 98.4 & 77.8 & 99.1
& \cellcolor{mazemax}\textbf{99.6} \\

& & 128 & & 97.5
& \cellcolor{mazemax}\textbf{99.8}
& 96.0 & 98.9
& \cellcolor{mazemax}\textbf{98.8}
& \cellcolor{mazemax}\textbf{99.0}
& 75.9 & 99.3 & 99.6 \\
& & 256 & & \cellcolor{mazemax}\textbf{98.4}
& 99.8
& 95.8
& 99.0
& 98.3 & 98.5 & 84.4
& \cellcolor{mazemax}\textbf{99.5}
& 99.6 \\
& & 512 & & 98.1
& 99.8
& 95.0
& 99.0
& 98.3 & 98.9 & 83.6 & 99.1 & 99.6 \\
\midrule
\multirow{10}{*}{\textbf{Maze15}} & \multirow[c]{10}{=}{\centering 88.25} & 1 & \multirow[c]{10}{=}{\centering 22.3} & 45.7 & 52.6 & 52.6 & 52.6 & 49.8 & 49.8 & \cellcolor{mazemax}\textbf{49.8} & 49.8 & \cellcolor{mazebest}\textbf{95.9} \\
 &  & 2 &  & 52.2 & 61.2 & 73.3 & 77.5 & 54.9 & 54.9 & 32.0 & 75.7 & 98.9 \\
 &  & 4 &  & 58.1 & 67.5 & 81.1 & 80.5 & 62.5 & 62.9 & 33.7 & 80.4 & 99.1 \\
 &  & 8 &  & 66.8 & 73.3 & 84.8 & 85.6 & 69.6 & 69.8 & 35.2 & 85.8 & 99.1 \\
 &  & 16 &  & 74.6 & 81.1 & 84.3 & 90.4 & 76.2 & 77.4 & 38.5 & 88.2 & 99.1 \\
 &  & 32 &  & 82.4 & 86.8 & \cellcolor{mazemax}\textbf{87.4} & 94.4 & 83.7 & 84.4 & 39.6 & 93.1 & 99.1 \\
 &  & 64 &  & 87.0 & 90.8 & 85.4 & \cellcolor{mazebest}\textbf{96.4} & 88.3 & 88.6 & 39.4 & 94.0 & \cellcolor{mazemax}\textbf{99.2} \\
 &  & 128 &  & 90.8 & 93.2 & 81.6 & \cellcolor{mazemax}\textbf{96.6} & 91.8 & 92.1 & 39.4 & \mazedual{97.8} & 99.2 \\
 &  & 256 &  & 91.5 & \cellcolor{mazebest}\textbf{95.0} & 79.7 & 96.1 & 93.7 & 94.2 & 39.9 & 97.8 & 99.2 \\
 &  & 512 &  & \cellcolor{mazemax}{\textbf{92.6}} & \cellcolor{mazemax}\textbf{95.4} & 78.4 & 95.9 & \mazedual{95.3} & \mazedual{96.2} & 42.1 & 97.1 & 99.2 \\
\midrule
\multirow{10}{*}{\textbf{Maze20}}
& \multirow[c]{10}{=}{\centering 258.81}
& 1
& \multirow[c]{10}{=}{\centering 32.2}
& 33.5
& 33.6
& 33.6
& 33.6
& 35.6
& 35.6
& \cellcolor{mazemax}\textbf{35.3}
& 35.6
& 71.4 \\

& & 2 &
& 37.3
& 37.8
& 41.7
& 62.4
& 41.2
& 42.6
& 30.8
& 49.3
& 88.1 \\

& & 4 &
& 44.0
& 40.8
& 68.3
& 70.1
& 48.2
& 48.8
& 28.7
& 54.0
& 91.3 \\

& & 8 &
& 51.0
& 49.4
& 84.0
& 76.5
& 55.8
& 54.1
& 29.5
& 60.4
& 92.1 \\

& & 16 &
& 60.8
& 56.0
& 89.0
& 79.9
& 63.0
& 64.9
& 30.1
& 67.0
& 92.4 \\

& & 32 &
& 67.5
& 63.8
& 90.0
& 83.3
& 70.6
& 71.7
& 30.0
& 73.1
& 92.5 \\

& & 64 &
& 77.3
& 73.2
& \cellcolor{mazemax}\textbf{90.4}
& 85.9
& 80.2
& 80.1
& 29.7
& 87.4
& 92.5 \\

& & 128 &
& 81.8
& 79.4
& 90.0
& 88.4
& 83.7
& 86.7
& 31.1
& 94.7
& 92.6 \\

& & 256 &
& \cellcolor{mazemax}\textbf{86.5}
& 82.5
& 87.2
& \cellcolor{mazemax}\textbf{92.1}
& 87.9
& 89.1
& 31.2
& \cellcolor{mazebest}\textbf{98.3}
& 92.6 \\

& & 512 &
& 86.3
& \cellcolor{mazemax}\textbf{87.1}
& 84.1
& 91.5
& \cellcolor{mazemax}\textbf{90.1}
& \cellcolor{mazemax}\textbf{90.6}
& 32.8
& \cellcolor{mazemax}\textbf{99.1}
& \cellcolor{mazemax}\textbf{92.9} \\
\bottomrule
\end{tabular}
\endgroup
\end{table}

\section{Experiments}
\label{sec:experiments}

Our experiments are designed to evaluate the \emph{reasoning capability} of continuous flow-based models rather than their general language modeling ability. We therefore focus on structured, non-linguistic tasks in two complementary reasoning domains: maze planning (Section \ref{sec:maze-reasoning}) and Sudoku (Section \ref{sec:sudoku-reasoning}). The maze datasets provide a controlled family of planning problems whose difficulty increases with maze size, allowing us to study not only how performance varies with the number of denoising steps but also how it depends on task difficulty and model capacity. Sudoku-Extreme provides a substantially different constraint-satisfaction problem and allows us to test whether the behavior of FLM extends beyond path planning. Finally, motivated by our theoretical analysis, we conduct a superposition analysis of FLM's intermediate states to examine whether the information they retain contributes to subsequent reasoning (Section \ref{sec:superposition-analysis}).

\subsection{Experimental Setup}
\label{sec:experimental-setup}

\paragraph{Benchmarks.}
For maze reasoning, we construct Maze10, Maze15, and Maze20 based on the maze-generation procedure in \citet{zhang2026vfscale}. With a consistent task formulation, these datasets provide a controlled setting for studying reasoning as maze size and path-planning difficulty increase.
We additionally evaluate on Sudoku-Extreme, which requires completing a $9\times9$ board that satisfies all row, column, and sub-grid constraints. To assess distributional generalization, we further test models trained on Sudoku-Extreme directly on Standard Sudoku without additional fine-tuning.
Construction procedures, splits, and detailed statistics for all datasets are provided in Appendix~\ref{app:dataset-details}.

\paragraph{Baselines and Evaluation.}
We compare FLM with representative discrete diffusion models, including D3PM
\citep{austin2021structured}, SEDD \citep{lou2023discrete}, MDLM
\citep{sahoo2024simple}, and ReMDM \citep{wang2026remasking}. Among these, ReMDM modifies the reverse transition to allow previously generated tokens
to be remasked and resampled, thereby permitting error correction. For masked discrete diffusion models, we additionally evaluate
confidence-based remasking, which revisits low-confidence tokens, and
entropy-ordered decoding, which prioritizes low-entropy positions without
revising committed tokens. A Transformer with chain-of-thought \citep{wei2022chain} serves as an additional autoregressive baseline for both maze planning and Sudoku. To ensure a fair comparison, all models use the same backbone and parameter count, with matched denoising steps for FLM and discrete diffusion methods.
Across all comparisons, we use sequence accuracy as the
primary metric, which counts a prediction as correct only when the complete
solution sequence matches the answer.

\paragraph{Implementation Details.}
For both maze and Sudoku tasks, FLM uses a bidirectional Transformer
with fixed input conditions. We train on Gaussian-to-one-hot
interpolations using a clean-token cross-entropy objective, optimize
with AdamW \citep{loshchilov2017fixing}, and evaluate EMA weights. Inference uses Euler integration
on a uniform time grid followed by tokenwise argmax.
All methods are adapted from their corresponding original codebases. Detailed implementations of FLM and all comparison baselines,
including architecture configurations, training settings, and
sampling procedures, are provided in
Appendix~\ref{app:Implementation-details}.

\begin{figure}[t]
    \centering
    \includegraphics[width=\linewidth]{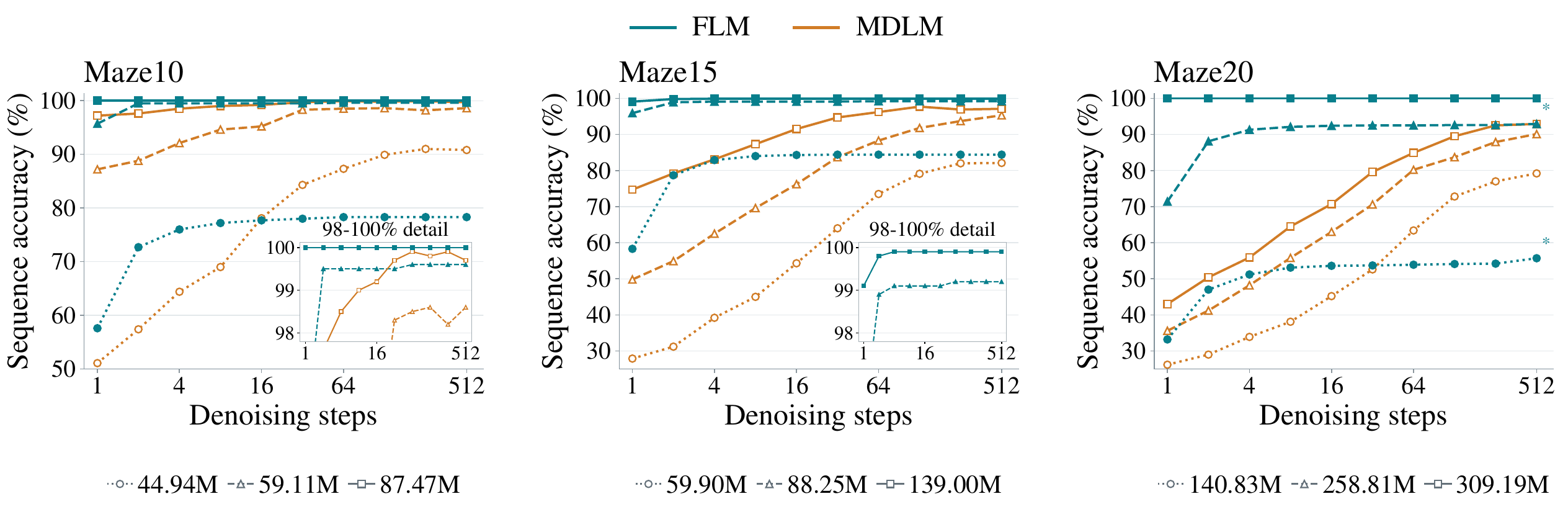}
    \caption{
\textbf{Sequence accuracy across denoising steps and model capacities.}
FLM and MDLM are evaluated on Maze10 (left), Maze15 (center), and
Maze20 (right), with three model sizes per task. Colors distinguish
methods, while marker shapes and line styles distinguish model sizes;
parameter counts are listed below each panel. Insets enlarge the
$98$ -- $100$\% sequence accuracy range.
}   
    \label{fig:maze-capacity-curves}
\end{figure}

\begin{figure}[t]
    \centering
    \includegraphics[width=\linewidth]{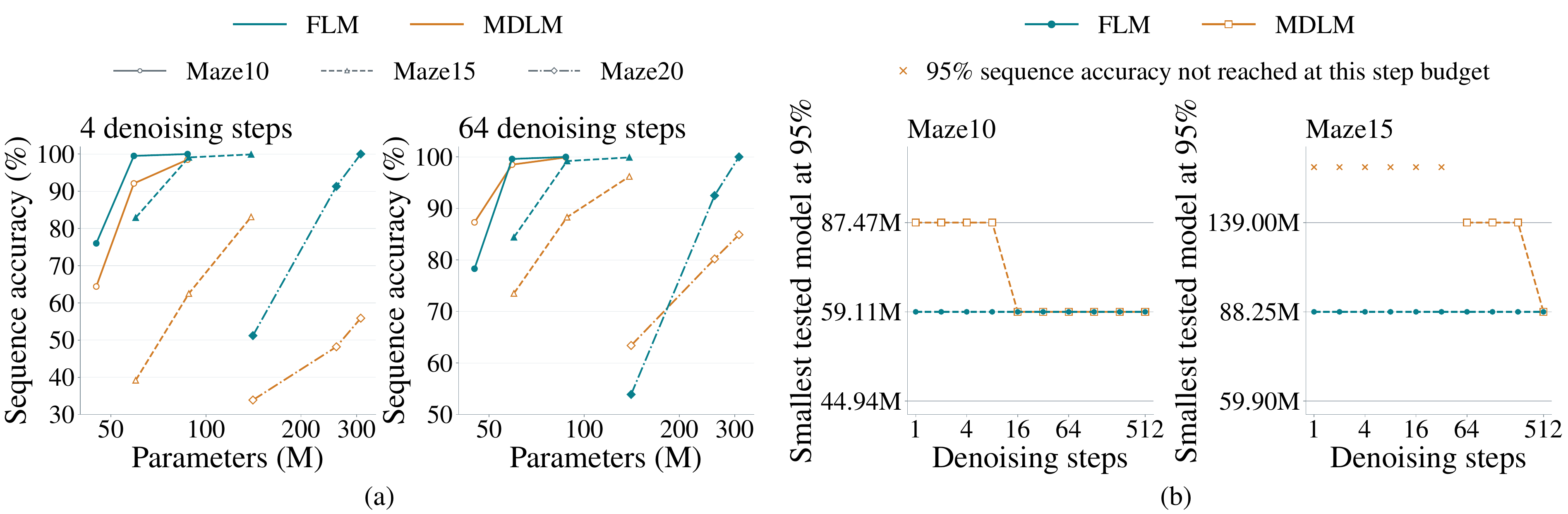}
    \caption{
\textbf{Model capacity and denoising efficiency across maze difficulties.}
(a) Sequence accuracy of FLM and MDLM on Maze10, Maze15, and Maze20 at 4 and 64 denoising steps.
(b) Smallest tested model reaching 95\% sequence accuracy at each step budget on Maze10 and Maze15.
Parameter counts in (a) use a logarithmic axis, while those in (b) are displayed as categorical levels.
Crosses indicate that no tested model reaches the target at the corresponding step budget.
}
    \label{fig:maze-capacity-fixed-steps}
\end{figure}

\subsection{Maze Planning: Reasoning Efficiency across Task Difficulties}
\label{sec:maze-reasoning}

Maze provides our main controlled benchmark for studying how continuous and discrete generative reasoning behave as planning difficulty increases. We compare against baselines and study the effect of model capacity. Additional results on wall-clock inference efficiency are given in Appendix~\ref{app:inference-efficiency}.

\paragraph{Comprehensive comparison with baselines.}
Table~\ref{tab:maze-main} compares FLM against discrete diffusion, enhanced sampling, and autoregressive baselines. FLM achieves the highest reported one-step sequence accuracy on all three tasks and maintains a clear few-step advantage on Maze15 and Maze20. These results highlight FLM's reasoning efficiency under limited denoising budgets.

\paragraph{Effect of model capacity.}
We compare FLM and MDLM across three model sizes per maze task (Figure~\ref{fig:maze-capacity-curves}). At comparable model sizes, FLM maintains higher accuracy as maze difficulty increases in several evaluated settings, with a larger gap over MDLM on Maze15 than on Maze10 (Figure~\ref{fig:maze-capacity-fixed-steps}(a)). FLM can also reach the same accuracy target with a smaller model:
on Maze15 at 64 denoising steps, the smallest tested models reaching 95\% accuracy have 88.25M parameters for FLM and 139.00M for MDLM
(Figure~\ref{fig:maze-capacity-fixed-steps}(b)).
These results show an advantage for FLM on harder maze tasks under limited model capacity.
We provide full results across all model sizes and denoising steps in Appendix~\ref{app:maze-capacity}.

\definecolor{sudokuyellow}{HTML}{FFF2CB}

% Highlight individual values without changing row height
\newcommand{\sudbest}[1]{%
    \begingroup
    \setlength{\fboxsep}{1pt}%
    \smash{\colorbox{sudokuyellow}{#1}}%
    \endgroup
}

\begin{table*}[t]
    \centering
    \caption{
    \textbf{Sequence accuracy (\%) on Sudoku-Extreme and Standard Sudoku
    under selected numbers of sampling steps.}
    Each entry reports Sudoku-Extreme / Standard Sudoku.
    All models have 86.88M parameters and are trained on Sudoku-Extreme;
    Standard Sudoku evaluates zero-shot transfer without further fine-tuning.
    \capmark{mazemax}{Yellow} highlights the highest accuracy for each dataset at each step
    budget, including ties.
    }
    \label{tab:sudoku-results}
    \setlength{\tabcolsep}{3pt}

    \begin{tabular*}{\textwidth}{
        @{\extracolsep{\fill}}lcccccc@{}
    }
        \toprule
        Method
        & \multicolumn{6}{c}{Evaluation Sampling Steps} \\
        \cmidrule(lr){2-7}
        & 1 & 2 & 16 & 64 & 256 & 512 \\
        \midrule

        Transformer + CoT
        & \multicolumn{6}{c}{2.9/3.5} \\

        \midrule

        D3PM
        & 11.2/19.5
        & 14.0/22.9
        & 32.0/38.6
        & 40.3/45.7
        & 42.3/46.1
        & 43.1/47.9 \\

        SEDD
        & 18.5/25.0
        & 21.2/28.2
        & 37.1/43.3
        & 46.8/50.8
        & 48.7/51.1
        & 47.9/53.1 \\

        SEDD + Ent.
        & 18.8/25.5
        & 36.2/\sudbest{41.4}
        & \sudbest{52.4}/\sudbest{57.9}
        & 53.6/57.8
        & 53.8/58.1
        & 54.3/57.5 \\

        SEDD + Conf.
        & 18.8/25.5
        & 26.4/32.1
        & 44.4/50.3
        & 56.2/60.6
        & 57.7/61.8
        & 57.4/62.0 \\

        MDLM
        & 19.1/27.1
        & 21.7/29.5
        & 40.0/45.0
        & 47.8/52.0
        & 47.6/53.3
        & 49.2/53.0 \\

        MDLM + Ent.
        & 19.2/27.1
        & 33.5/39.5
        & 51.2/55.6
        & 52.5/57.0
        & 51.8/58.3
        & 53.6/57.6 \\

        MDLM + Conf.
        & 19.2/27.1
        & 26.5/33.6
        & 47.6/51.5
        & \sudbest{60.2}/\sudbest{64.1}
        & \sudbest{62.3}/\sudbest{65.6}
        & \sudbest{62.2}/\sudbest{64.3} \\

        ReMDM
        & 19.2/27.1
        & 21.7/29.6
        & 38.8/44.4
        & 46.4/52.4
        & 52.0/55.0
        & 53.8/59.2 \\

        \midrule

        FLM
        & \sudbest{21.0}/\sudbest{29.5}
        & \sudbest{36.7}/\sudbest{41.4}
        & 46.3/49.5
        & 43.9/50.5
        & 46.0/51.4
        & 45.5/51.3 \\

        \bottomrule
    \end{tabular*}
\end{table*}

\subsection{Sudoku: Few-Step Constraint Reasoning and Generalization}
\label{sec:sudoku-reasoning}

We next test whether the same few-step advantage appears in a qualitatively different reasoning domain. Unlike maze planning, Sudoku requires resolving a large collection of mutually interacting row, column, and sub-grid constraints, providing a complementary test of structured reasoning.

\paragraph{Few-step refinement on Sudoku-Extreme.}
Table~\ref{tab:sudoku-results} shows that FLM improves rapidly within the
first two sampling steps. On Sudoku-Extreme, FLM achieves $21.0\%$
sequence accuracy at one step and $36.7\%$ at two steps, compared with
$19.2\%$ and $36.2\%$ for the strongest discrete baseline at each budget.
Further sampling yields more limited gains for FLM, whereas enhanced
discrete samplers achieve higher accuracy from 16 steps onward.
These results support FLM's effectiveness under few step denoising.

\paragraph{Generalization to Standard Sudoku.}
We further evaluate models trained on Sudoku-Extreme directly on Standard
Sudoku without additional fine-tuning. At one sampling step, FLM achieves
$29.5\%$ sequence accuracy, exceeding the strongest discrete result of
$27.1\%$. 
% At two steps, FLM reaches $41.4\%$, matching SEDD + Ent.
The single-step advantage therefore persists under this distribution
shift.

\subsection{Superposition Analysis}
\label{sec:superposition-analysis}
\begin{wrapfigure}{r}{0.55\textwidth}
    \centering
    \vspace{-1.5em}
    \includegraphics[width=\linewidth]
    {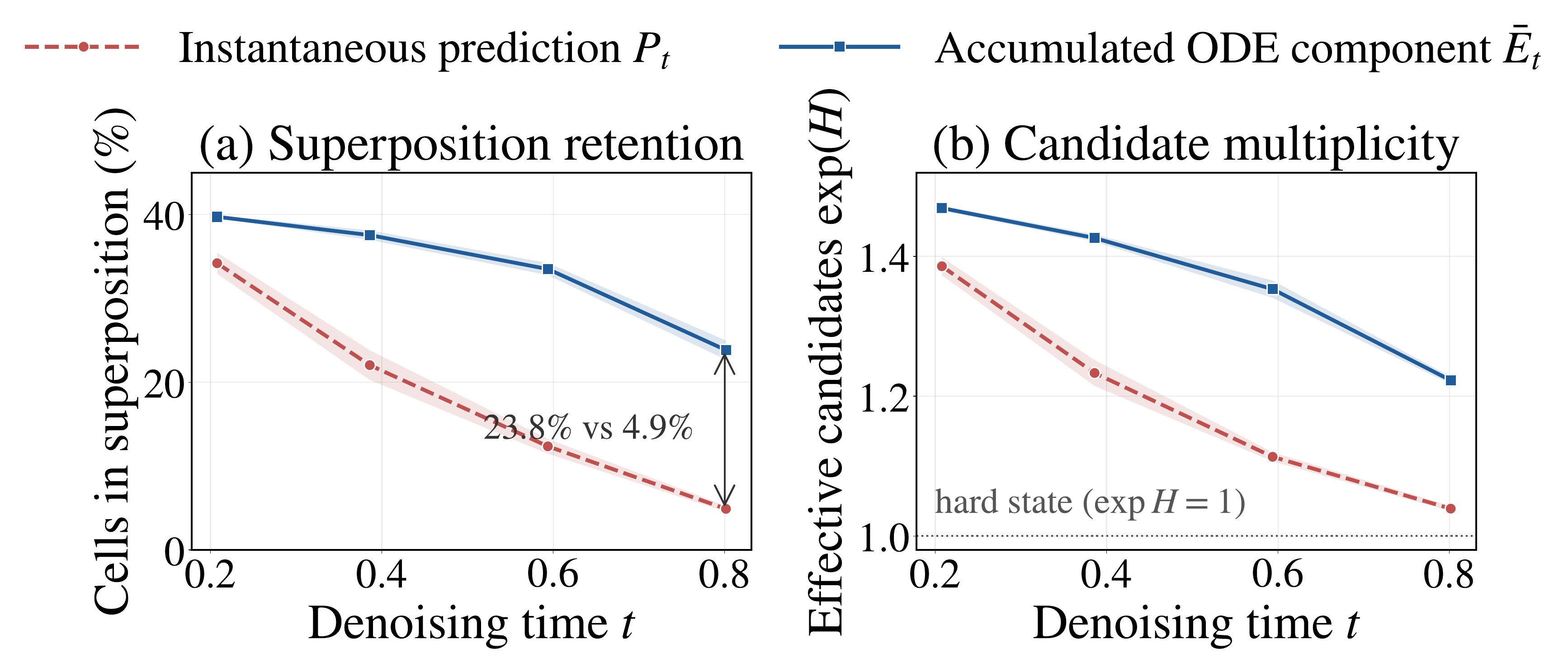}
    \caption{\textbf{Candidate retention during FLM denoising.}
We compare the instantaneous clean prediction $\mP_t$ with the accumulated
component $\overline{\mathbf E}_t$. (a) shows the fraction of ambiguous
cells whose second-largest digit probability is at least $0.05$. (b)
shows the effective number of candidates, $\exp(H)$. The accumulated component retains more non-top-1 information
than the instantaneous prediction.}
    \label{fig:superposition}
    \vspace{-1.2em}
\end{wrapfigure}
\paragraph{Continuous states retain non-top-1 candidate information.}
We first examine the accumulated denoising component
$\overline{\mathbf E}_t$ from
Theorem~\ref{thm:flm-candidate-superposition} and compare it with the
instantaneous clean prediction $\mP_t$ on Sudoku-Extreme. For each analyzed
blank cell, we normalize the mass over the nine digit coordinates and denote
the resulting distribution by $p_t(d)$, where $d\in\{1,\ldots,9\}$. We measure
% its uncertainty using the Shannon entropy
% $
% H(p_t)=-\sum_{d=1}^{9}p_t(d)\log p_t(d),
% $
% and report $\exp(H(p_t))$ as the effective number of candidates. 
% As shown in Figure~\ref{fig:superposition}, $\mP_t$ quickly concentrates on
% one digit, whereas $\overline{\mathbf E}_t$ retains more mass on secondary
% candidates. At $t=0.80$, $23.8\%$ of cells in
% $\overline{\mathbf E}_t$ retain a second candidate with probability at least
% $0.05$, compared with only $4.9\%$ in $\mP_t$. Its larger effective candidate
% count further indicates that the accumulated component remains less
% concentrated. These observations are consistent with our theoretical result
% that a persistent continuous state can retain a nontrivial candidate
% superposition.
We quantify the concentration of $p_t$ through the effective number of candidates, $\exp(H(p_t))$, where
$
H(p_t)=-\sum_{d=1}^{9}p_t(d)\log p_t(d)
$
is the Shannon entropy computed with natural logarithms. For a uniform distribution over $k$ candidate digits, $H(p_t)=\log k$ and hence $\exp(H(p_t))=k$, giving this measure a direct interpretation as an effective support size.
As shown in Figure~\ref{fig:superposition}, $\mP_t$ quickly concentrates on one digit, whereas $\overline{\mathbf E}_t$ retains more mass on secondary candidates. At $t=0.80$, $23.8\%$ of cells in $\overline{\mathbf E}_t$ retain a second candidate with probability at least $0.05$, compared with only $4.9\%$ in $\mP_t$. The larger effective candidate count of $\overline{\mathbf E}_t$ further indicates that its mass remains more broadly distributed across candidate digits. These observations are consistent with our theoretical result.

\begin{wrapfigure}{r}{0.53\textwidth}
    \centering
    \vspace{-1.0em}
    \includegraphics[width=\linewidth]
    {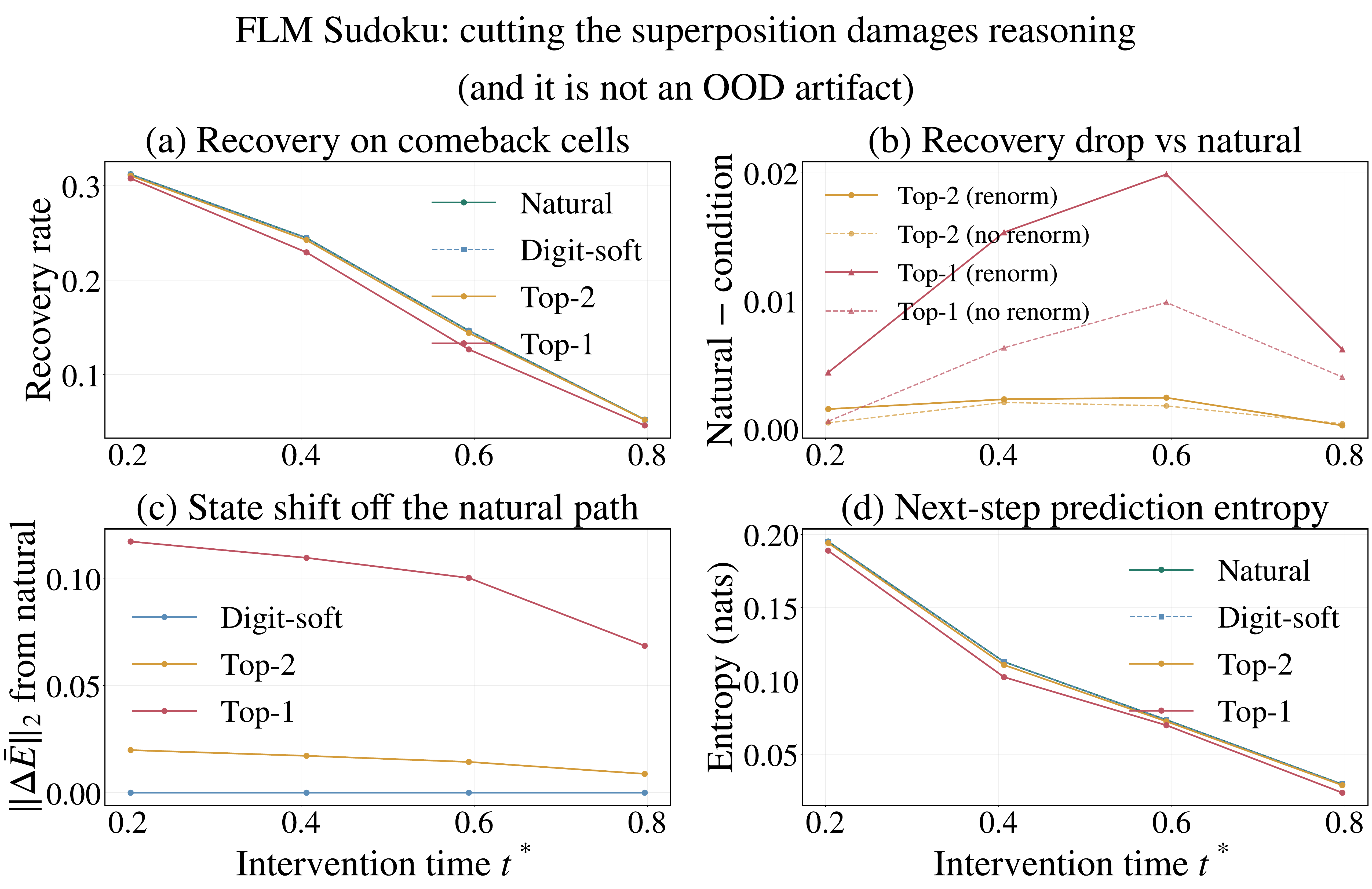}
    \caption{\textbf{Intermediate-state intervention.}
    Removing non-top-1 candidate information reduces recovery,
    while Digit-soft causes little degradation.}
    \label{fig:flm-superposition}
    \vspace{-1.5em}
\end{wrapfigure}
\paragraph{Intermediate-state intervention.}
Candidate retention alone does not establish that the model uses the retained
information. We therefore intervene at an intermediate time $t^\star$ while
keeping the input, initial noise, model, and subsequent denoising schedule
fixed. 

We compare three intermediate-state interventions against an unmodified baseline. \emph{Natural} leaves the accumulated continuous state unchanged. \emph{Digit-soft} masks locally infeasible digit
coordinates but retains and renormalizes all locally feasible digits. \emph{Top-$2$} and \emph{Top-$1$} apply the same candidate filtering operation while retaining only the two or one highest-scoring feasible digits, respectively. Thus, Top-$2$ and Top-$1$ progressively remove non-top-1 candidate information from the intermediate state. The Digit-soft condition controls for the effect of the projection and renormalization operation itself: it modifies the state but does not remove any locally feasible candidate. We also compare renormalized and non-renormalized top-$k$ interventions. The same degradation pattern in both settings suggests that the result is not caused solely by rescaling the retained coordinates. We measure \emph{cell-level recovery rate} on ambiguous blank cells whose correct digit ranks second among the nine digit coordinates in the unmodified accumulated state at $t^\star$. Recovery rate measures final cell accuracy on cells whose
correct digit ranks second before intervention.

As shown in Figure~\ref{fig:flm-superposition}, cell-level
recovery rate decreases as more candidate information is removed,
$
\mathrm{Recovery}_{\mathrm{Natural}}
\geq
\mathrm{Recovery}_{\mathrm{Top\text{-}2}}
>
\mathrm{Recovery}_{\mathrm{Top\text{-}1}}.
$
Digit-soft causes little degradation, whereas Top-$2$ and especially Top-$1$
reduce recovery. Moreover, Top-$1$ lowers rather than increases the entropy of
the next-step prediction, making increased prediction uncertainty an unlikely
explanation for its poorer performance. These controlled interventions provide empirical evidence that non-top-1 candidate information retained in
the continuous state contributes to subsequent denoising.

\begin{wraptable}{r}{0.41\textwidth}
\vspace{-1.5em}
\centering
\small
\setlength{\tabcolsep}{4pt}
\caption{
\textbf{Test F1-score for solution-path prediction along
an eight-step denoising trajectory.}
Bold highlights the step with the largest Raw--Hard gap.
}
\label{tab:decoder-time-course}
\begin{tabular}{rcccc}
\toprule
$k$ & Raw & Top-2 & Hard & Raw $-$ Hard \\
\midrule
0 & 0.3799 & 0.4072 & 0.4090 & $-0.0291$ \\
1 & 0.3927 & 0.4066 & 0.4050 & $-0.0123$ \\
2 & 0.4765 & 0.4218 & 0.4091 & $+0.0674$ \\
3 & 0.6484 & 0.4674 & 0.4565 & $+0.1919$ \\
\textbf{4}
& \textbf{0.8080}
& \textbf{0.6614}
& \textbf{0.5929}
& $\mathbf{+0.2151}$ \\
5 & 0.8989 & 0.8331 & 0.8172 & $+0.0817$ \\
6 & 0.9158 & 0.9109 & 0.9135 & $+0.0024$ \\
7 & 0.9189 & 0.9184 & 0.9189 & $\approx 0$ \\
8 & 0.9189 & 0.9188 & 0.9189 & $\approx 0$ \\
\bottomrule
\end{tabular}
\vspace{-1.5em}
\end{wraptable}

\paragraph{Decoding task information from intermediate states.}
Motivated by Theorem~\ref{thm:flm-decoder-accessibility},
we study whether FLM states make solution-path
information more accessible than hard token assignments during denoising. On Maze15, we train decoders for each representation and step of a frozen FLM's $N=8$
denoising trajectory. Each decoder receives the maze input
and the corresponding state representation and predicts
whether each cell belongs to the solution path.
We evaluate F1-score from pooled test-cell predictions
at a fixed probability threshold of 0.5.
Experimental details are in
Appendix~\ref{app:flm-decoder-experiments}.

We compare raw states $Z_k$, Top-2 representations, and
one-hot argmax assignments (Hard). Top-2 retains the two
largest entries of $\operatorname{softmax}(Z_k)$ at each
position, setting the remaining coordinates to zero.
Table~\ref{tab:decoder-time-course} shows that the raw-state
advantage is concentrated at intermediate steps.
The largest Raw--Hard gap occurs at $k=4$, where Raw achieves
an F1-score of \textbf{0.8080}, compared with 0.6614 for
Top-2 and 0.5929 for Hard.
By $k=7,8$, Raw and Hard agree to four decimal places.
These results suggest that solution-path information is
more accessible to the tested decoders in intermediate
continuous states, before comparable readout performance
is achieved from hard token assignments.

\section{Related Work}

% \paragraph{Discrete diffusion language models.}
% Discrete diffusion language models have advanced through improved corruption processes, training objectives, and sampling strategies. D3PM \citep{austin2021structured} and SEDD \citep{lou2023discrete} established discrete transition and score-based formulations, while MDLM \citep{sahoo2024simple} simplified masked-denoising training. LLaDA \citep{nie2026large} and Dream \citep{ye2025dream} demonstrated competitive language modeling at scale. Confidence- and entropy-based sampling guide token updates, while ReMDM \citep{wang2026remasking} enables iterative refinement through remasking. Other variants explore semantic clusters, subtokens, or partial noising. Despite these developments, denoising still proceeds through categorical intermediate states.

\paragraph{Continuous diffusion language models.}
Diffusion-LM \citep{li2022diffusion} and DiffuSeq \citep{gong2022diffuseq} denoise embeddings, while TextLDM \citep{jiang2026textldm} uses autoencoding latents. Plaid \citep{gulrajani2023likelihood}, ELF \citep{hu2026elf}, and LangFlow \citep{chen2026langflow} advance continuous diffusion and flow matching for language modeling. Other representations include relaxed vocabulary states \citep{han2023ssd,mahabadi2024tess,tae2025tess}, Riemannian categorical distributions \citep{cheng2024categorical,davis2024fisher}, and continuous binary encodings in Analog Bits \citep{chen2022analog} and CoBit \citep{batzolis2026cobit}.

\paragraph{Reasoning in continuous latent spaces.}
Continuous latent reasoning performs intermediate computation in hidden representations without expressing every step as text. Coconut \citep{hao2024training} feeds hidden states back as input embeddings, forming continuous thought chains that support reasoning over multiple possible next steps. Recurrent approaches refine internal representations through repeated updates. The Hierarchical Reasoning Model (HRM) \citep{wang2025hierarchicalreasoningmodel} couples two recurrent modules at different timescales, while the Tiny Recursive Model (TRM) \citep{jolicoeur2025less} uses a single small network to recursively update latent reasoning features and answer representations. These methods highlight the potential of iterative continuous computation for reasoning tasks.

\section{Conclusion}
% We studied how continuous and discrete state spaces affect the reasoning
% efficiency of diffusion language models. 
In this work, we have examined the possible advantages of a continuous state space over a discrete one for diffusion language models.
Experiments on maze planning and Sudoku show that FLM can achieve higher sequence accuracy with few denoising steps. Our superposition analysis and intermediate-state experiments support a mechanism through which continuous states preserve candidate information and use it in subsequent reasoning.

\paragraph{Limitations and future work.}
This work focuses on structured reasoning tasks and does not evaluate
general language modeling ability. Further study on text-based reasoning
benchmarks is needed to assess whether the observed benefits extend to
natural-language reasoning. Another direction is to investigate how these
benefits change when scaling FLMs to larger model capacities.

\section*{Acknowledgements}
This work was supported in part by the Hector Foundation, by the Max Planck ETH Center for Learning Systems, and by EPSRC Turing AI World-Leading Research Fellowship No. EP/X040062/1 and EPSRC AI Hub on Mathematical Foundations of Intelligence: An "Erlangen Programme" for AI No. EP/Y028872/1.

\bibliography{iclr2027_conference}

@article{jolicoeur2025less,
  title={Less is more: Recursive reasoning with tiny networks},
  author={Jolicoeur-Martineau, Alexia},
  journal={arXiv preprint arXiv:2510.04871},
  year={2025}
}

@article{loshchilov2017fixing,
  title={Fixing weight decay regularization in adam},
  author={Loshchilov, Ilya and Hutter, Frank and others},
  journal={arXiv preprint arXiv:1711.05101},
  volume={5},
  number={5},
  pages={5},
  year={2017}
}

@article{wang2026remasking,
  title={Remasking discrete diffusion models with inference-time scaling},
  author={Wang, Guanghan and Schiff, Yair and Sahoo, Subham and Kuleshov, Volodymyr},
  journal={Advances in Neural Information Processing Systems},
  volume={38},
  pages={147282--147339},
  year={2026}
}

@article{wei2022chain,
  title={Chain-of-thought prompting elicits reasoning in large language models},
  author={Wei, Jason and Wang, Xuezhi and Schuurmans, Dale and Bosma, Maarten and Xia, Fei and Chi, Ed and Le, Quoc V and Zhou, Denny and others},
  journal={Advances in neural information processing systems},
  volume={35},
  pages={24824--24837},
  year={2022}
}

@inproceedings{zhang2026vfscale,
  title={VFScale: Intrinsic reasoning through verifier-free test-time scalable diffusion model},
  author={Zhang, Tao and Pan, Jia-Shu and Feng, Ruiqi and Wu, Tailin},
  booktitle={International Conference on Learning Representations},
  volume={2026},
  pages={59758--59787},
  year={2026}
}

@article{davis2024fisher,
  title={Fisher flow matching for generative modeling over discrete data},
  author={Davis, Oscar and Kessler, Samuel and Petrache, Mircea and Ceylan, {\.I}smail {\.I} and Bronstein, Michael and Bose, Avishek J},
  journal={Advances in Neural Information Processing Systems},
  volume={37},
  pages={139054--139084},
  year={2024}
}

@article{cheng2024categorical,
  title={Categorical flow matching on statistical manifolds},
  author={Cheng, Chaoran and Li, Jiahan and Peng, Jian and Liu, Ge},
  journal={Advances in Neural Information Processing Systems},
  volume={37},
  pages={54787--54819},
  year={2024}
}

@inproceedings{tae2025tess,
  title={Tess 2: A large-scale generalist diffusion language model},
  author={Tae, Jaesung and Ivison, Hamish and Kumar, Sachin and Cohan, Arman},
  booktitle={Proceedings of the 63rd Annual Meeting of the Association for Computational Linguistics (Volume 1: Long Papers)},
  pages={21171--21188},
  year={2025}
}

@inproceedings{mahabadi2024tess,
  title={Tess: Text-to-text self-conditioned simplex diffusion},
  author={Mahabadi, Rabeeh Karimi and Ivison, Hamish and Tae, Jaesung and Henderson, James and Beltagy, Iz and Peters, Matthew E and Cohan, Arman},
  booktitle={Proceedings of the 18th Conference of the European Chapter of the Association for Computational Linguistics (Volume 1: Long Papers)},
  pages={2347--2361},
  year={2024}
}

@article{chen2022analog,
  title={Analog bits: Generating discrete data using diffusion models with self-conditioning},
  author={Chen, Ting and Zhang, Ruixiang and Hinton, Geoffrey},
  journal={arXiv preprint arXiv:2208.04202},
  year={2022}
}

@article{gulrajani2023likelihood,
  title={Likelihood-based diffusion language models},
  author={Gulrajani, Ishaan and Hashimoto, Tatsunori B},
  journal={Advances in Neural Information Processing Systems},
  volume={36},
  pages={16693--16715},
  year={2023}
}

@inproceedings{han2023ssd,
  title={Ssd-lm: Semi-autoregressive simplex-based diffusion language model for text generation and modular control},
  author={Han, Xiaochuang and Kumar, Sachin and Tsvetkov, Yulia},
  booktitle={Proceedings of the 61st Annual Meeting of the Association for Computational Linguistics (Volume 1: Long Papers)},
  pages={11575--11596},
  year={2023}
}

@article{jiang2026textldm,
  title={TextLDM: Language Modeling with Continuous Latent Diffusion},
  author={Jiang, Jiaxiu and Ren, Jingjing and Li, Wenbo and Wang, Bo and Sun, Haoze and Yang, Yijun and Liu, Jianhui and Zhang, Yanbing and Zheng, Shenghe and Zhang, Yuan and others},
  journal={arXiv preprint arXiv:2605.07748},
  year={2026}
}

@article{gong2022diffuseq,
  title={Diffuseq: Sequence to sequence text generation with diffusion models},
  author={Gong, Shansan and Li, Mukai and Feng, Jiangtao and Wu, Zhiyong and Kong, LingPeng},
  journal={arXiv preprint arXiv:2210.08933},
  year={2022}
}

@article{li2022diffusion,
  title={Diffusion-lm improves controllable text generation},
  author={Li, Xiang and Thickstun, John and Gulrajani, Ishaan and Liang, Percy S and Hashimoto, Tatsunori B},
  journal={Advances in neural information processing systems},
  volume={35},
  pages={4328--4343},
  year={2022}
}

@article{chen2026langflow,
  title={LangFlow: Continuous Diffusion Rivals Discrete in Language Modeling},
  author={Chen, Yuxin and Liang, Chumeng and Sui, Hangke and Guo, Ruihan and Cheng, Chaoran and You, Jiaxuan and Liu, Ge},
  journal={arXiv preprint arXiv:2604.11748},
  year={2026}
}

@article{hu2026elf,
  title={Elf: Embedded language flows},
  author={Hu, Keya and Qiu, Linlu and Lu, Yiyang and Zhao, Hanhong and Li, Tianhong and Kim, Yoon and Andreas, Jacob and He, Kaiming},
  journal={arXiv preprint arXiv:2605.10938},
  year={2026}
}

@article{lee2026flow,
  title={Flow map language models: One-step language modeling via continuous denoising},
  author={Lee, Chanhyuk and Yoo, Jaehoon and Agarwal, Manan and Shah, Sheel and Huang, Jerry and Raghunathan, Aditi and Hong, Seunghoon and Boffi, Nicholas M and Kim, Jinwoo},
  journal={arXiv preprint arXiv:2602.16813},
  year={2026}
}

@article{hao2024training,
  title={Training large language models to reason in a continuous latent space},
  author={Hao, Shibo and Sukhbaatar, Sainbayar and Su, DiJia and Li, Xian and Hu, Zhiting and Weston, Jason and Tian, Yuandong},
  journal={arXiv preprint arXiv:2412.06769},
  year={2024}
}

@article{yao2023tree,
  title={Tree of thoughts: Deliberate problem solving with large language models},
  author={Yao, Shunyu and Yu, Dian and Zhao, Jeffrey and Shafran, Izhak and Griffiths, Tom and Cao, Yuan and Narasimhan, Karthik},
  journal={Advances in neural information processing systems},
  volume={36},
  pages={11809--11822},
  year={2023}
}

@article{ye2025dream,
  title={Dream 7b: Diffusion large language models},
  author={Ye, Jiacheng and Xie, Zhihui and Zheng, Lin and Gao, Jiahui and Wu, Zirui and Jiang, Xin and Li, Zhenguo and Kong, Lingpeng},
  journal={arXiv preprint arXiv:2508.15487},
  year={2025}
}

@article{nie2026large,
  title={Large language diffusion models},
  author={Nie, Shen and Zhu, Fengqi and You, Zebin and Zhang, Xiaolu and Ou, Jingyang and Hu, Jun and Zhou, Jun and Lin, Yankai and Wen, Ji-Rong and Li, Chongxuan},
  journal={Advances in Neural Information Processing Systems},
  volume={38},
  pages={50608--50646},
  year={2026}
}

@article{sahoo2024simple,
  title={Simple and effective masked diffusion language models},
  author={Sahoo, Subham S and Arriola, Marianne and Schiff, Yair and Gokaslan, Aaron and Marroquin, Edgar and Chiu, Justin T and Rush, Alexander and Kuleshov, Volodymyr},
  journal={Advances in Neural Information Processing Systems},
  volume={37},
  pages={130136--130184},
  year={2024}
}

@article{lou2023discrete,
  title={Discrete diffusion modeling by estimating the ratios of the data distribution},
  author={Lou, Aaron and Meng, Chenlin and Ermon, Stefano},
  journal={arXiv preprint arXiv:2310.16834},
  year={2023}
}

@article{austin2021structured,
  title={Structured denoising diffusion models in discrete state-spaces},
  author={Austin, Jacob and Johnson, Daniel D and Ho, Jonathan and Tarlow, Daniel and Van Den Berg, Rianne},
  journal={Advances in neural information processing systems},
  volume={34},
  pages={17981--17993},
  year={2021}
}

@article{he2026reasoning,
  title={Reasoning with Latent Tokens in Diffusion Language Models},
  author={He, Andre and Welleck, Sean and Fried, Daniel},
  journal={arXiv preprint arXiv:2602.03769},
  year={2026}
}

@article{zhu2026reasoning,
  title={Reasoning by superposition: A theoretical perspective on chain of continuous thought},
  author={Zhu, Hanlin and Hao, Shibo and Hu, Zhiting and Jiao, Jiantao and Russell, Stuart J and Tian, Yuandong},
  journal={Advances in Neural Information Processing Systems},
  volume={38},
  pages={79931--79963},
  year={2026}
}

@article{xu2025toward,
  title={Toward large reasoning models: A survey of reinforced reasoning with large language models},
  author={Xu, Fengli and Hao, Qianyue and Shao, Chenyang and Zong, Zefang and Li, Yu and Wang, Jingwei and Zhang, Yunke and Wang, Jingyi and Lan, Xiaochong and Gong, Jiahui and others},
  journal={Patterns},
  volume={6},
  number={10},
  year={2025},
  publisher={Elsevier}
}

@inproceedings{patil2025advancing,
  title={Advancing reasoning in large language models: Promising methods and approaches},
  author={Patil, Avinash and Jadon, Aryan},
  booktitle={International Conference on Computational Intelligence and Soft Computing},
  pages={284--298},
  year={2025},
  organization={Springer}
}

@misc{roosCategoricalFlowMaps2026,
    title = {Categorical {Flow} {Maps}},
    url = {http://arxiv.org/abs/2602.12233},
    doi = {10.48550/arXiv.2602.12233},
    urldate = {2026-08-16},
    publisher = {arXiv},
    author = {Roos, Daan and Davis, Oscar and Eijkelboom, Floor and Bronstein, Michael and Welling, Max and Ceylan, İsmail İlkan and Ambrogioni, Luca and Meent, Jan-Willem van de},
    month = feb,
    year = {2026},
    note = {arXiv:2602.12233 [cs.LG]},
}

@misc{potaptchik2026discreteflowmaps,
      title={Discrete Flow Maps}, 
      author={Peter Potaptchik and Jason Yim and Adhi Saravanan and Peter Holderrieth and Eric Vanden-Eijnden and Michael S. Albergo},
      year={2026},
      eprint={2604.09784},
      archivePrefix={arXiv},
      primaryClass={stat.ML},
      url={https://arxiv.org/abs/2604.09784}, 
}

@misc{wang2025hierarchicalreasoningmodel,
      title={Hierarchical Reasoning Model}, 
      author={Guan Wang and Jin Li and Yuhao Sun and Xing Chen and Changling Liu and Yue Wu and Meng Lu and Sen Song and Yasin Abbasi Yadkori},
      year={2025},
      eprint={2506.21734},
      archivePrefix={arXiv},
      primaryClass={cs.AI},
      url={https://arxiv.org/abs/2506.21734}, 
}

@article{batzolis2026cobit,
  title={CoBit: Language Modeling with Bitstream Diffusion},
  author={Batzolis, Georgios and Girolami, Mark and Ambrogioni, Luca},
  journal={arXiv preprint arXiv:2605.07013},
  year={2026}
}
\bibliographystyle{iclr2027_conference}

\newpage
\appendix

% Requires \usepackage{hyperref}
% Replace the previous \appitem definition with this one.
\newcommand{\appitem}[1]{\item[\ref{#1}] \nameref{#1}}

\section*{Appendix Overview}

The appendix provides training and sampling algorithms, theoretical
proofs, and additional experimental details and results.

\begin{itemize}
    \appitem{sec:algorithms}
    \appitem{app:related-work}
    \appitem{app:theorem proof}
    \begin{itemize}
        \appitem{app:flm-superposition-proof}
        \begin{itemize}
            \appitem{app:subsec:candidate-vertices}
            \appitem{app:subsec:superposition-geometry}
            \appitem{app:subsec:complete-graph-accessibility}
            \appitem{app:subsec:candidate-convex-decomposition}
            \appitem{app:subsec:flm-ode-derivation}
            \appitem{app:subsec:flm-integral-proof}
        \end{itemize}

        \appitem{app:history-mechanism}
        \appitem{app:history-refinement}
        \appitem{app:flm-decoder-proof}
    \end{itemize}

    \appitem{app:experiments}
    \begin{itemize}
        \appitem{app:dataset-details}
        \begin{itemize}
            \appitem{app:maze-datasets}
            \appitem{app:sudoku-datasets}
        \end{itemize}

        \appitem{app:Implementation-details}
        \begin{itemize}
            \appitem{app:architecture-optimization}
            \appitem{app:inference-details}
            \appitem{app:flm-decoder-experiments}
        \end{itemize}

        \appitem{app:full-state-intervention}
        \appitem{app:maze-capacity}
        \appitem{app:inference-efficiency}
        
    \end{itemize}
\end{itemize}

\newpage

\section{Training and Sampling Algorithms}
\label{sec:algorithms}

Algorithms~\ref{alg:cfm-training} and~\ref{alg:cfm-sampling} summarize the
training and sampling procedures used by our conditional FLM.
Algorithm~\ref{alg:cfm-training} constructs a continuous noisy answer state
by interpolating between Gaussian noise and the clean tokenized target, and
trains the model to predict the clean answer tokens. At inference time,
Algorithm~\ref{alg:cfm-sampling} starts from Gaussian noise, evolves the
continuous answer state using the learned flow for $N$ denoising evaluations,
and discretizes the final state into a token sequence and the corresponding
path.
\begin{algorithm}[t]
\caption{Training the conditional flow matching model}
\label{alg:cfm-training}
\begin{algorithmic}[1]
\Require Dataset $\mathcal D$, tokenizer $\operatorname{tok}$, model
$\mP_\theta$
\Repeat
    \State Sample $(\gG,\vq,Y)\sim\mathcal D$
    \State $\mathbf c\gets
    \operatorname{tok}(\operatorname{ser}_{\mathrm{in}}(\gG,\vq))
    \mathbin\Vert[\mathrm{SEP}]$
    \State $\mathbf x\gets
    \operatorname{tok}(\operatorname{ser}_{\mathrm{out}}(Y))
    \mathbin\Vert[\mathrm{EOS}]$
    \State $\mathbf C_{\gG,\vq}\gets\operatorname{OH}(\mathbf c)$,
    $\mZ_1\gets\operatorname{OH}(\mathbf x)$
    \State Sample $\mZ_0\sim\mathcal N(\vzero,\mI)$ and
    $t\sim\operatorname{Unif}[0,1]$
    \State $\mZ_t\gets(1-t)\mZ_0+t\mZ_1$
    \State $\mP\gets
    \mP_\theta(\mZ_t,t,\mathbf C_{\gG,\vq})$
    \State Update $\theta$ using
    $\operatorname{CE}(\mZ_1,\mP)$ over answer positions
\Until{convergence}
\end{algorithmic}
\end{algorithm}

\begin{algorithm}[t]
\caption{Sampling from the conditional flow matching model}
\label{alg:cfm-sampling}
\begin{algorithmic}[1]
\Require Graph $\gG$, query $\vq$, tokenizer $\operatorname{tok}$, model
$\mP_\theta$, answer length $M_x$, number of evaluations $N$
\State $\mathbf c\gets
\operatorname{tok}(\operatorname{ser}_{\mathrm{in}}(\gG,\vq))
\mathbin\Vert[\mathrm{SEP}]$
\State $\mathbf C_{\gG,\vq}\gets\operatorname{OH}(\mathbf c)$
\State Sample
$\widehat{\mZ}_{t_0}\sim\mathcal N(\vzero,\mI)$ in
$\mathbb R^{M_x\times K}$
\State Set $t_k=k/N$ for $k=0,\ldots,N$
\For{$k=0,\ldots,N-1$}
    \State $\mP_k\gets
    \mP_\theta(\widehat{\mZ}_{t_k},t_k,\mathbf C_{\gG,\vq})$
    \State $\widehat{\vv}_{t_k}\gets
    (\mP_k-\widehat{\mZ}_{t_k})/(1-t_k)$
    \State $\widehat{\mZ}_{t_{k+1}}\gets
    \widehat{\mZ}_{t_k}+(t_{k+1}-t_k)\widehat{\vv}_{t_k}$
\EndFor
\State $\widehat x_i\gets
\operatorname*{arg\,max}_{a\in\mathcal V_{\mathrm{tok}}}
(\widehat{\mZ}_{t_N})_{i,a}$ for $i=1,\ldots,M_x$
\State $\widehat s_{\mathrm{out}}\gets
\operatorname{detok}(\widehat x_1,\ldots,\widehat x_{M_x})$
\State \Return $\widehat Y\gets
\operatorname{parse}_{\mathrm{path}}(\widehat s_{\mathrm{out}})$
\end{algorithmic}
\end{algorithm}

\section{Additional Related Work}
\label{app:related-work}

\paragraph{Discrete diffusion language models.}
Discrete diffusion language models have advanced through improved corruption
processes, training objectives, and sampling strategies.
D3PM \citep{austin2021structured} and SEDD \citep{lou2023discrete}
introduced discrete transition and score-based formulations, while
MDLM \citep{sahoo2024simple} simplified masked-denoising training.
LLaDA \citep{nie2026large} and Dream \citep{ye2025dream}
demonstrated competitive language modeling at scale.
Sampling strategies based on confidence or entropy guide token selection,
while ReMDM \citep{wang2026remasking} allows previously generated tokens
to be remasked and resampled.
These methods perform denoising through categorical intermediate states.

\section{Theoretical Results}
\label{app:theorem proof}
\subsection{Proof of Theorem \ref{thm:flm-candidate-superposition}}
\label{app:flm-superposition-proof}

\begin{lemma}[Instantaneous candidate-sequence superposition]
\label{lem:instantaneous-candidate-superposition}
For each candidate token sequence
$\mathbf s\in\mathcal V_{\mathrm{tok}}^{M_x}$, define
\begin{equation}
\pi_{\mathbf s}(t)
=
\prod_{j=1}^{M_x}(\mP_t)_{j,s_j}.
%\label{eq:instantaneous-candidate-weight}
\end{equation}
We call $\mathbf s$ a positive-weight candidate at time $t$ if
$\pi_{\mathbf s}(t)>0$, and define the support of the candidate-weight
distribution as
\begin{equation}
\operatorname{supp}(\pi_t)
=
\left\{
\mathbf s\in\mathcal S_{M_x}
:
\pi_{\mathbf s}(t)>0
\right\}.
% \label{eq:candidate-weight-support}
\end{equation}
Then
\begin{equation}
\pi_{\mathbf s}(t)\geq0,
\qquad
\sum_{\mathbf s\in\mathcal S_{M_x}}
\pi_{\mathbf s}(t)=1,
\label{eq:instantaneous-weights-normalized}
\end{equation}
and the following exact decomposition holds:
\begin{equation}
\mP_t
=
\sum_{\mathbf s\in\mathcal V_{\mathrm{tok}}^{M_x}}
\pi_{\mathbf s}(t)\mathbf E^{(\mathbf s)}.
% \label{eq:instantaneous-candidate-decomposition}
\end{equation}
Hence, the instantaneous clean prediction is a convex superposition of
candidate token-sequence vertices. If two positive-weight candidates differ
at one or more token positions, then $\mP_t\notin\mathcal H$ and the
superposition is nontrivial.
\end{lemma}

Lemma~\ref{lem:instantaneous-candidate-superposition} is an instantaneous
geometric statement about the clean prediction $\mP_t$. The Theorem \ref{thm:flm-candidate-superposition} establishes the stronger trajectory-level result: the FLM ODE integrates
these instantaneous superpositions over time and retains them in its
persistent continuous state. We address this question using the FLM dynamics.
This appendix provides complete proofs of
Lemma~\ref{lem:instantaneous-candidate-superposition} and
Theorem~\ref{thm:flm-candidate-superposition}. We first represent complete
hard-token sequences as discrete vertices and establish the geometry of their
positive-weight superpositions. We then show that global attention provides
every answer position with access to the complete input graph. Next, we prove
that the position-wise clean prediction of FLM automatically admits a convex
decomposition over candidate token sequences. Finally, we derive the FLM ODE
and its integral representation, prove the main theorem, and compare the
resulting continuous state with the persistent hard state of absorbing
discrete diffusion.

\subsubsection{Candidate Sequences as Discrete Vertices}
\label{app:subsec:candidate-vertices}

Let the answer contain $M_x$ token positions, and let the tokenizer
vocabulary $\mathcal V_{\mathrm{tok}}$ have size $K$. A candidate token
sequence
\begin{equation*}
\mathbf s=(s_1,\ldots,s_{M_x})
\in\mathcal V_{\mathrm{tok}}^{M_x}
\end{equation*}
corresponds to the one-hot matrix
\begin{equation}
\mathbf E^{(\mathbf s)}
=
\begin{bmatrix}
\ve_{s_1}^{\top}\\
\vdots\\
\ve_{s_{M_x}}^{\top}
\end{bmatrix}
\in\{0,1\}^{M_x\times K}.
\label{eq:app-candidate-vertex}
\end{equation}
Each row of $\mathbf E^{(\mathbf s)}$ contains exactly one entry equal to one
and all remaining entries equal to zero. Define the set of complete
hard-token sequence states as
\begin{equation}
\mathcal H
=
\left\{
\mathbf E^{(\mathbf s)}:
\mathbf s\in\mathcal V_{\mathrm{tok}}^{M_x}
\right\}.
\label{eq:app-hard-sequence-set}
\end{equation}
Also define the product simplex
\begin{equation}
\mathcal Q
=
(\Delta^{K-1})^{M_x}
=
\left\{
\mathbf P\in\mathbb R^{M_x\times K}:
P_{j,a}\geq0,\quad
\sum_{a=1}^{K}P_{j,a}=1
\text{ for every }j
\right\}.
\label{eq:app-product-simplex}
\end{equation}
The set $\mathcal H$ is precisely the vertex set of $\mathcal Q$, and
\begin{equation}
\mathcal Q=\operatorname{conv}(\mathcal H).
\label{eq:app-product-simplex-hull}
\end{equation}
The proof of Lemma~\ref{lem:instantaneous-candidate-superposition} below also
proves
Eq.~\eqref{eq:app-product-simplex-hull} by giving an explicit convex
decomposition for every $\mathbf P\in\mathcal Q$. Strictly speaking,
$\mathbf E^{(\mathbf s)}$ is a vertex of the product simplex $\mathcal Q$,
rather than a vertex of the unconstrained vector space
$\mathbb R^{M_x\times K}$. A complete discrete hard-token state can take
only an element of $\mathcal H$. Following the notation of
Section~\ref{sec:problem-setup}, the FLM state
$\mZ_t\in\mathbb R^{M_x\times K}$ lies in the ambient continuous space; the
denoised component defined below lies in $\mathcal Q$.

\subsubsection{Geometry of Multi-Candidate Superposition}
\label{app:subsec:superposition-geometry}

Consider $R\geq2$ distinct candidate vertices
$\mathbf E^{(1)},\ldots,\mathbf E^{(R)}\in\mathcal H$ with strictly
positive weights
\begin{equation}
\pi_m>0,
\qquad
\sum_{m=1}^{R}\pi_m=1.
\label{eq:app-positive-candidate-weights}
\end{equation}
Define their continuous superposition as
\begin{equation}
\mathbf S(\boldsymbol\pi)
=
\sum_{m=1}^{R}\pi_m\mathbf E^{(m)}.
\label{eq:app-geometric-superposition}
\end{equation}

\begin{lemma}[Nontriviality of a multi-candidate convex combination]
\label{lem:app-geometric-superposition}
The state $\mathbf S(\boldsymbol\pi)$ belongs to the convex hull of the
candidate vertices:
\begin{equation}
\mathbf S(\boldsymbol\pi)
\in
\operatorname{conv}
\{\mathbf E^{(1)},\ldots,\mathbf E^{(R)}\}
\subseteq
\mathbb R^{M_x\times K}.
\label{eq:app-superposition-in-hull}
\end{equation}
If two positive-weight candidates differ at one or more token positions,
then
\begin{equation}
\mathbf S(\boldsymbol\pi)\notin\mathcal H.
\label{eq:app-superposition-not-hard}
\end{equation}
\end{lemma}

\begin{proof}
Equation~\eqref{eq:app-superposition-in-hull} follows directly from the
definition of a convex hull. Suppose candidates $r$ and $s$ select different
tokens at position $j$:
\begin{equation*}
s_j^{(r)}=a,
\qquad
s_j^{(s)}=b,
\qquad
a\neq b,
\end{equation*}
with $\pi_r,\pi_s>0$. Then
\begin{equation}
\mathbf S(\boldsymbol\pi)_{j,a}
\geq\pi_r>0,
\qquad
\mathbf S(\boldsymbol\pi)_{j,b}
\geq\pi_s>0.
\label{eq:app-two-positive-token-coordinates}
\end{equation}
Thus, row $j$ of $\mathbf S(\boldsymbol\pi)$ contains at least two positive
coordinates and is not one-hot. Since every row of every state in
$\mathcal H$ must be one-hot,
$\mathbf S(\boldsymbol\pi)\notin\mathcal H$.
\end{proof}

Geometrically, the convex hull of two candidate vertices is a line segment,
the convex hull of three affinely independent candidate vertices is a
triangle, and additional affinely independent vertices generate a
higher-dimensional simplex. If the selected candidate vertices are affinely
independent and all their weights are strictly positive, the superposition
lies in the relative interior of that simplex. If the vertices are not
affinely independent, their positive-weight combination still lies in their
convex hull and, whenever the candidates disagree at a token position, is not
a hard vertex. It need not, however, lie in the relative interior of the
convex hull generated by every listed vertex.

\subsubsection{Accessibility of the Complete Graph}
\label{app:subsec:complete-graph-accessibility}

For a graph node $v$, let $\mathcal I_v$ denote the token positions associated
with $v$ in the conditioning sequence. Visibility of the complete graph means
that
\begin{equation}
\mathcal I_v\neq\varnothing,
\qquad
\forall v\in\gV.
\label{eq:app-node-occurrence}
\end{equation}
In attention layer $\ell$ and head $h$, let the attention weight from answer
position $\bar j$ to an unmasked position $i$ be
\begin{equation}
A_{\bar j i}^{(\ell,h)}(t)
=
\frac{
\exp(\ell_{\bar j i}^{(\ell,h)}(t))
}{
\displaystyle
\sum_{r\in\mathcal I_{\mathrm{real}}}
\exp(\ell_{\bar j r}^{(\ell,h)}(t))
}.
\label{eq:app-global-attention-weight}
\end{equation}

\begin{lemma}[Global attention covers every input node]
\label{lem:app-global-attention-coverage}
Suppose that all attention logits are finite, attention dropout is disabled at
inference time, and only padding positions are masked. Then, for every
$i\in\mathcal I_v$,
\begin{equation}
A_{\bar j i}^{(\ell,h)}(t)>0.
\label{eq:app-positive-global-attention}
\end{equation}
Consequently, every answer position has an attention route to every graph
node.
\end{lemma}

\begin{proof}
The exponential of every finite logit is strictly positive, and the
denominator of Eq.~\eqref{eq:app-global-attention-weight} is a finite sum of
strictly positive terms. Hence, every unmasked position receives a strictly
positive attention weight. Because $\mathcal I_v$ is nonempty for every
$v\in\gV$, each answer position has a positive-attention route to every graph
node.
\end{proof}

This lemma establishes only structural accessibility: the model can use the
complete graph to form candidate evidence. It does not guarantee that the
learned network uses every node, nor does it guarantee that a candidate token
sequence represents a valid graph path.

\subsubsection{Proof of Instantaneous Candidate-Sequence Superposition}
\label{app:subsec:candidate-convex-decomposition}

Along the FLM trajectory, let
\begin{equation}
\mP_t
=
\mP_\theta(\mZ_t,t,\gG,\vq)
\in(\Delta^{K-1})^{M_x}
\label{eq:app-clean-token-distribution}
\end{equation}
denote the position-wise clean-token distribution. For every
$\mathbf s=(s_1,\ldots,s_{M_x})\in\mathcal V_{\mathrm{tok}}^{M_x}$, define
\begin{equation}
\pi_{\mathbf s}(t)
=
\prod_{j=1}^{M_x}(\mP_t)_{j,s_j}.
\label{eq:app-candidate-product-weight}
\end{equation}

\begin{proof}[Proof of Lemma~\ref{lem:instantaneous-candidate-superposition}]
Since every row of $\mP_t$ is a probability distribution,
\begin{align}
\sum_{\mathbf s}\pi_{\mathbf s}(t)
&=
\sum_{s_1,\ldots,s_{M_x}}
\prod_{j=1}^{M_x}(\mP_t)_{j,s_j}
\notag\\
&=
\prod_{j=1}^{M_x}
\sum_{a\in\mathcal V_{\mathrm{tok}}}
(\mP_t)_{j,a}
\notag\\
&=1.
\label{eq:app-product-weights-sum}
\end{align}
Now consider coordinate $(j,a)$ on the right-hand side of
Eq.~\eqref{eq:instantaneous-candidate-decomposition}:
\begin{align}
&\sum_{\mathbf s}
\pi_{\mathbf s}(t)
(\mathbf E^{(\mathbf s)})_{j,a}
\notag\\
&\quad=
\sum_{\mathbf s:s_j=a}
\prod_{r=1}^{M_x}(\mP_t)_{r,s_r}
\notag\\
&\quad=
(\mP_t)_{j,a}
\prod_{r\neq j}
\sum_{b\in\mathcal V_{\mathrm{tok}}}
(\mP_t)_{r,b}
\notag\\
&\quad=(\mP_t)_{j,a}.
\label{eq:app-coordinate-decomposition}
\end{align}
The equality holds for every $j$ and $a$, which proves
Eq.~\eqref{eq:instantaneous-candidate-decomposition}.

If two positive-weight candidates $\mathbf s\neq\mathbf s'$ differ at
position $j$, let $s_j=a\neq b=s'_j$. The decomposition gives
\begin{equation}
(\mP_t)_{j,a}\geq\pi_{\mathbf s}(t)>0,
\qquad
(\mP_t)_{j,b}\geq\pi_{\mathbf s'}(t)>0.
\label{eq:app-instantaneous-two-token-masses}
\end{equation}
Thus, row $j$ of $\mP_t$ is not one-hot, and hence
$\mP_t\notin\mathcal H$. This proves the nontrivial-superposition statement.
\end{proof}

Equation~\eqref{eq:app-candidate-product-weight} constructs a factorized
candidate-sequence distribution induced by the position-wise token marginals.
It proves that $\mP_t$ automatically belongs to the convex hull of all
candidate token-sequence vertices. It does not imply that these sequences are
valid graph paths, nor that this factorized distribution captures the general
joint dependencies of complete paths. Because the same coordinate-wise
construction applies to every $\mathbf P\in\mathcal Q$, it also proves
Eq.~\eqref{eq:app-product-simplex-hull}.

\subsubsection{Derivation of the FLM ODE}
\label{app:subsec:flm-ode-derivation}

Training uses the linear probability path
\begin{equation}
\mZ_t
=
(1-t)\mZ_0+t\mZ_1,
\label{eq:app-linear-probability-path}
\end{equation}
where $\mZ_0$ is Gaussian noise and $\mZ_1$ is the clean one-hot
answer state. Differentiating with respect to $t$ gives
\begin{equation}
\frac{d\mZ_t}{dt}
=
\mZ_1-\mZ_0.
\label{eq:app-linear-path-derivative}
\end{equation}
Rearranging Eq.~\eqref{eq:app-linear-probability-path} yields
\begin{equation}
\mZ_0
=
\frac{\mZ_t-t\mZ_1}{1-t},
\qquad t<1.
\label{eq:app-noise-from-interpolation}
\end{equation}
Substituting Eq.~\eqref{eq:app-noise-from-interpolation} into
Eq.~\eqref{eq:app-linear-path-derivative}, we obtain
\begin{align}
\frac{d\mZ_t}{dt}
&=
\mZ_1
-
\frac{\mZ_t-t\mZ_1}{1-t}
\notag\\
&=
\frac{\mZ_1-\mZ_t}{1-t}.
\label{eq:app-clean-endpoint-velocity}
\end{align}
At sampling time, the true $\mZ_1$ is unknown. The model therefore
replaces it with the clean-state prediction $\mP_t$, defining the learned
velocity field
\begin{equation}
\mathbf F_\theta
(\mZ_t,t,\mathbf C_{\gG,\vq})
=
\frac{\mP_t-\mZ_t}{1-t},
\label{eq:app-learned-flm-velocity}
\end{equation}
and the sampling ODE
\begin{equation}
\frac{d\mZ_t}{dt}
=
\frac{\mP_t-\mZ_t}{1-t},
\qquad 0\leq t<1.
\label{eq:app-flm-sampling-ode}
\end{equation}
Under an ideal clean-state predictor, $\mP_t$ approximates
\begin{equation*}
\mathbb E[
\mZ_1
\mid
\mZ_t,t,\gG,\vq
],
\end{equation*}
so Eq.~\eqref{eq:app-learned-flm-velocity} corresponds to the conditional
mean velocity. For a learned model,
Eq.~\eqref{eq:app-learned-flm-velocity} is the continuous velocity field
defined by its clean-state predictor. Since the denominator vanishes at
$t=1$, the following dynamics and integral identities are stated for $t<1$.

\subsubsection{Integral Representation and Proof of the Main Theorem}
\label{app:subsec:flm-integral-proof}

Rewrite Eq.~\eqref{eq:app-flm-sampling-ode} as
\begin{equation}
\frac{d\mZ_t}{dt}
+
\frac{\mZ_t}{1-t}
=
\frac{\mP_t}{1-t}.
\label{eq:app-linear-ode-form}
\end{equation}
This is a first-order linear identity along the realized trajectory. Although
$\mP_t$ depends on $\mZ_t$, the following integral identity remains
exact along every solution of the ODE. The integrating factor is
\begin{equation}
\mu(t)
=
\exp\left(\int\frac{1}{1-t}\,dt\right)
=
\frac{1}{1-t}.
\label{eq:app-integrating-factor}
\end{equation}
Multiplying Eq.~\eqref{eq:app-linear-ode-form} by $\mu(t)$ gives
\begin{equation}
\frac{1}{1-t}\frac{d\mZ_t}{dt}
+
\frac{\mZ_t}{(1-t)^2}
=
\frac{\mP_t}{(1-t)^2}.
\label{eq:app-scaled-linear-ode}
\end{equation}
By the product rule,
\begin{equation*}
\frac{d}{dt}
\left(\frac{\mZ_t}{1-t}\right)
=
\frac{1}{1-t}\frac{d\mZ_t}{dt}
+
\frac{\mZ_t}{(1-t)^2}.
\end{equation*}
Therefore,
\begin{equation}
\frac{d}{dt}
\left(\frac{\mZ_t}{1-t}\right)
=
\frac{\mP_t}{(1-t)^2}.
\label{eq:app-total-derivative-form}
\end{equation}
Integrating from $0$ to $t$ and applying the fundamental theorem of calculus
gives
\begin{equation}
\frac{\mZ_t}{1-t}
-\mZ_0
=
\int_0^t
\frac{\mP_u}{(1-u)^2}\,du.
\label{eq:app-integrated-total-derivative}
\end{equation}
Multiplying by $1-t$ yields
\begin{equation}
\mZ_t
=
(1-t)\mZ_0
+
(1-t)\int_0^t
\frac{\mP_u}{(1-u)^2}\,du.
\label{eq:app-flm-integral-representation}
\end{equation}
Equation~\eqref{eq:app-flm-integral-representation} decomposes the current
state into a residual contribution from the initial noise and an accumulated
contribution from the preceding clean predictions. Even though $\mP_u$
depends on $\mZ_u$ through the neural network,
Eq.~\eqref{eq:app-flm-integral-representation} is an exact integral identity
along the solution trajectory.

\begin{proof}[Proof of Theorem~\ref{thm:flm-candidate-superposition}]
Apply Lemma~\ref{lem:instantaneous-candidate-superposition} at time $u$ and
substitute its decomposition
\begin{equation}
\mP_u
=
\sum_{\mathbf s}
\pi_{\mathbf s}(u)\mathbf E^{(\mathbf s)},
\label{eq:app-timewise-candidate-decomposition}
\end{equation}
into Eq.~\eqref{eq:app-flm-integral-representation}:
\begin{align}
\mZ_t
&=(1-t)\mZ_0
\notag\\
&\quad+
\sum_{\mathbf s}
\left[
(1-t)\int_0^t
\frac{\pi_{\mathbf s}(u)}{(1-u)^2}\,du
\right]
\mathbf E^{(\mathbf s)}.
\label{eq:app-expanded-integral-representation}
\end{align}
Define
\begin{equation}
\lambda_{\mathbf s}(t)
=
(1-t)\int_0^t
\frac{\pi_{\mathbf s}(u)}{(1-u)^2}\,du.
\label{eq:app-integrated-candidate-weight}
\end{equation}
Because $\pi_{\mathbf s}(u)\geq0$ and the integral kernel is strictly
positive, $\lambda_{\mathbf s}(t)\geq0$. Furthermore, since
$\sum_{\mathbf s}\pi_{\mathbf s}(u)=1$,
\begin{align}
\sum_{\mathbf s}\lambda_{\mathbf s}(t)
&=
(1-t)\int_0^t\frac{1}{(1-u)^2}\,du
\notag\\
&=
(1-t)\left(\frac{1}{1-t}-1\right)
\notag\\
&=t.
\label{eq:app-integrated-weights-sum}
\end{align}
For $t>0$, define
\begin{equation}
\overline\pi_{\mathbf s}(t)
=
\frac{\lambda_{\mathbf s}(t)}{t}.
\label{eq:app-normalized-integrated-weight}
\end{equation}
Then $\overline\pi_{\mathbf s}(t)\geq0$ and
$\sum_{\mathbf s}\overline\pi_{\mathbf s}(t)=1$. Hence,
\begin{equation}
\mZ_t
=
(1-t)\mZ_0
+t\sum_{\mathbf s}
\overline\pi_{\mathbf s}(t)
\mathbf E^{(\mathbf s)}.
\label{eq:app-normalized-state-decomposition}
\end{equation}
Defining
\begin{equation}
\overline{\mathbf E}_t
=
\sum_{\mathbf s}
\overline\pi_{\mathbf s}(t)
\mathbf E^{(\mathbf s)}
\label{eq:app-denoised-superposition}
\end{equation}
gives the decomposition stated in the theorem.

Suppose that two distinct candidates $\mathbf s\neq\mathbf s'$ have positive
weights on subsets of time of positive measure. Because the integral kernel
is strictly positive,
\begin{equation}
\overline\pi_{\mathbf s}(t)>0,
\qquad
\overline\pi_{\mathbf s'}(t)>0.
\label{eq:app-two-positive-integrated-weights}
\end{equation}
Distinct sequences differ at some position $j$. Let
$s_j=a\neq b=s_j'$. Then
\begin{equation}
(\overline{\mathbf E}_t)_{j,a}
\geq\overline\pi_{\mathbf s}(t)>0,
\qquad
(\overline{\mathbf E}_t)_{j,b}
\geq\overline\pi_{\mathbf s'}(t)>0.
\label{eq:app-superposition-two-token-masses}
\end{equation}
Thus, row $j$ of $\overline{\mathbf E}_t$ is not one-hot. By
Lemma~\ref{lem:app-geometric-superposition},
\begin{equation}
\overline{\mathbf E}_t\notin\mathcal H.
\label{eq:app-denoised-component-not-hard}
\end{equation}
Therefore, contributions from multiple candidate sequences are accumulated
through the same FLM ODE trajectory into a nontrivial geometric
superposition.
\end{proof}

\subsection{Proof of Theorem~\ref{thm:history-mechanism-full-influence}}
\label{app:history-mechanism}

\begin{proof}
Throughout, the context, time \(t\), model parameters, and attention
mask are fixed. All derivatives are with respect to \(\lambda\).
We suppress the dependence on \(\lambda\) where convenient.

\medskip
\noindent\emph{Hard-state preservation.}
The intervention modifies only coordinate \((j,v)\). Since \(v\)
is a strict non-top-1 candidate and \(0\leq\lambda\leq h_t\),
\[
(\mathbf U_t(\lambda))_{j,v}
=(\mZ_t)_{j,v}-(h_t-\lambda)
\leq(\mZ_t)_{j,v}
<(\mZ_t)_{j,x_j}.
\]
Every other coordinate remains unchanged. Hence the set of maximizing
coordinates in each row is unchanged, and fixed tie-breaking gives
\[
\Pi(\mathbf U_t(\lambda))=\Pi(\mZ_t)=\mathbf x.
\]
This proves the first equality in
Equation~\ref{eq:full-history-hard-preservation}.

\medskip
\noindent\emph{Exact derivative decomposition.}
For head \(b\) and receiving position \(i\), let
\(\mathcal I_i^b\) be the fixed set of visible positions. Define
\[
s_{im}^b
=\frac{\langle\mathbf q_i^b,\mathbf k_m^b\rangle}{\sqrt{d_b}},
\qquad
a_{im}^b
=\frac{\exp(s_{im}^b)}
{\sum_{\ell\in\mathcal I_i^b}\exp(s_{i\ell}^b)},
\qquad m\in\mathcal I_i^b.
\]
The product rule yields
\[
\dot s_{im}^b
=
\frac{\langle\dot{\mathbf q}_i^b,\mathbf k_m^b\rangle}
{\sqrt{d_b}}
+
\frac{\langle\mathbf q_i^b,\dot{\mathbf k}_m^b\rangle}
{\sqrt{d_b}}
=
\rho_{im}^{Q,b}+\rho_{im}^{K,b}.
\]
Differentiating the softmax gives
\begin{equation}
\dot a_{im}^b
=
a_{im}^b
\left(
\dot s_{im}^b
-\sum_{\ell\in\mathcal I_i^b}
a_{i\ell}^b\dot s_{i\ell}^b
\right).
\label{eq:app-history-softmax-derivative}
\end{equation}

Since
\(\mathbf o_i^b=\sum_{m\in\mathcal I_i^b}
a_{im}^b\mathbf v_m^b\),
\begin{equation}
\dot{\mathbf o}_i^b
=
\sum_{m\in\mathcal I_i^b}
\dot a_{im}^b\mathbf v_m^b
+
\sum_{m\in\mathcal I_i^b}
a_{im}^b\dot{\mathbf v}_m^b.
\label{eq:app-history-output-derivative}
\end{equation}
By the complete interface representation
\(M_t^{(\lambda)}=F_t(O(\lambda),\xi(\lambda))\),
the chain rule gives
\begin{equation}
\partial_\lambda M_t^{(\lambda)}
=
\sum_{b,i}
\langle\mathbf g_i^b,\dot{\mathbf o}_i^b\rangle
+
\langle\nabla_\xi F_t,\dot\xi\rangle.
\label{eq:app-history-chain-rule}
\end{equation}
Here all gradients are evaluated at the current intervened state,
and include the dependence through all subsequent model layers.

Using \(\eta_{im}^b
=\langle\mathbf g_i^b,\mathbf v_m^b\rangle\),
the attention-weight contribution at each head and position is
\[
\begin{aligned}
\sum_{m\in\mathcal I_i^b}\dot a_{im}^b\eta_{im}^b
&=
\mathbb E_{m\sim a_i^b}[\dot s_{im}^b\eta_{im}^b]
-
\mathbb E_{m\sim a_i^b}[\dot s_{im}^b]\,
\mathbb E_{m\sim a_i^b}[\eta_{im}^b]\\
&=
\operatorname{Cov}_{m\sim a_i^b}
(\rho_{im}^{Q,b},\eta_{im}^b)
+
\operatorname{Cov}_{m\sim a_i^b}
(\rho_{im}^{K,b},\eta_{im}^b).
\end{aligned}
\]
The value contribution is
\[
\left\langle
\mathbf g_i^b,\sum_{m\in\mathcal I_i^b}
a_{im}^b\dot{\mathbf v}_m^b
\right\rangle
=
\mathbb E_{m\sim a_i^b}
\langle\mathbf g_i^b,\dot{\mathbf v}_m^b\rangle.
\]
Substituting these identities into
Equation~\ref{eq:app-history-chain-rule} and summing over all heads
and receiving positions gives
\[
\partial_\lambda M_t^{(\lambda)}
=
\mathcal C_Q+\mathcal C_K+\mathcal C_V+\mathcal C_R
=\mathcal C(\lambda),
\]
with the contributions stated in
Equation~\ref{eq:full-history-path-contributions}.
This proves the second equality in
Equation~\ref{eq:full-history-hard-preservation}.
No sign restriction on the individual contributions is required.

\medskip
\noindent\emph{Cross-position influence.}
Suppose the maps preceding attention act independently at each
position. Because the intervention changes only position \(j\),
for \(i\neq j\) we have
\[
\dot{\mathbf q}_i^b=0,
\qquad
\dot{\mathbf k}_m^b=\dot{\mathbf v}_m^b=0
\quad\text{for }m\neq j.
\]
For \(j\in\mathcal I_i^b\), it follows that
\[
\dot s_{im}^b
=\mathbf 1\{m=j\}\rho_{ij}^{K,b}.
\]
Equation~\ref{eq:app-history-softmax-derivative} therefore becomes
\[
\dot a_{im}^b
=
a_{im}^b
\bigl(\mathbf 1\{m=j\}-a_{ij}^b\bigr)
\rho_{ij}^{K,b}.
\]
Substitution into
Equation~\ref{eq:app-history-output-derivative} yields
\[
\begin{aligned}
\dot{\mathbf o}_i^b
&=
a_{ij}^b\mathbf v_j^b\rho_{ij}^{K,b}
-
a_{ij}^b\rho_{ij}^{K,b}
\sum_{m\in\mathcal I_i^b}a_{im}^b\mathbf v_m^b
+
a_{ij}^b\dot{\mathbf v}_j^b\\
&=
a_{ij}^b(\mathbf v_j^b-\mathbf o_i^b)\rho_{ij}^{K,b}
+
a_{ij}^b\dot{\mathbf v}_j^b,
\end{aligned}
\]
which is Equation~\ref{eq:history-cross-position}.

\medskip
\noindent\emph{Local sensitivity and the integrated response.}
The smoothness assumption implies that
\(\lambda\mapsto M_t^{(\lambda)}\) is \(C^1\), with continuous
derivative \(\mathcal C(\lambda)\).
Thus, by the fundamental theorem of calculus,
\[
M_t^{(\lambda_1)}-M_t^{(\lambda_0)}
=
\int_{\lambda_0}^{\lambda_1}\mathcal C(s)\,ds
\qquad
(\lambda_0,\lambda_1\in[0,h_t]).
\]
Taking \(\lambda_0=0\) and \(\lambda_1=h_t\) proves
Equation~\ref{eq:full-history-integrated-effect}.

If \(\mathcal C(\lambda_*)\neq0\), continuity ensures that
\(\mathcal C\) has the same sign as \(\mathcal C(\lambda_*)\)
and magnitude at least
\(\lvert\mathcal C(\lambda_*)\rvert/2\)
in a sufficiently small neighborhood of \(\lambda_*\).
Consequently, every sufficiently small nonzero \(\delta\) with
\(\lambda_*+\delta\in[0,h_t]\) satisfies
\[
\left|
M_t^{(\lambda_*+\delta)}-M_t^{(\lambda_*)}
\right|
=
\left|
\int_{\lambda_*}^{\lambda_*+\delta}\mathcal C(s)\,ds
\right|
\geq
\frac{|\mathcal C(\lambda_*)|}{2}|\delta|
>0.
\]
The argument also covers one-sided admissible perturbations at
the endpoints. Together with hard-state preservation, this proves
the claimed local prediction sensitivity at fixed hard state.

\medskip
\noindent\emph{Hard-state approximation error and tightness.}
Continuity on the compact interval \([0,h_t]\) ensures that
\[
M_{\min}:=\min_{\lambda\in[0,h_t]}M_t^{(\lambda)},
\qquad
M_{\max}:=\max_{\lambda\in[0,h_t]}M_t^{(\lambda)}
\]
are attained, with \(\Omega_t=M_{\max}-M_{\min}\).
Since \(h_t>0\) and \(M_t^{(\lambda)}\) is \(C^1\), it is constant
on this interval if and only if
\(\mathcal C\equiv0\). Hence
\[
\Omega_t>0
\quad\Longleftrightarrow\quad
\mathcal C\not\equiv0.
\]

For any scalar readout \(H\) using only context, time, and the
hard state, its input is constant along the intervention path.
Write
\[
c_H:=H(\mathbf C_{\gG,\vq},t,\mathbf x).
\]
The triangle inequality gives
\[
\begin{aligned}
\Omega_t
&=|M_{\max}-M_{\min}|\\
&\leq |M_{\max}-c_H|+|M_{\min}-c_H|\\
&\leq
2\sup_{\lambda\in[0,h_t]}
|M_t^{(\lambda)}-c_H|.
\end{aligned}
\]
Moreover, both endpoint values lie in
\([M_{\min},M_{\max}]\), so
\[
|\Delta_t|
=|M_t^{(h_t)}-M_t^{(0)}|
\leq \Omega_t.
\]
Combining these inequalities establishes
Equation~\ref{eq:history-mechanism-hard-error}.

Finally, for this fixed intervention path, choose the constant
readout
\[
H_*(\cdot)=\frac{M_{\max}+M_{\min}}{2}.
\]
Every path value lies between \(M_{\min}\) and \(M_{\max}\),
and both extrema are attained. Therefore,
\[
\sup_{\lambda\in[0,h_t]}
\left|
M_t^{(\lambda)}
-H_*(\mathbf C_{\gG,\vq},t,\mathbf x)
\right|
=\frac{M_{\max}-M_{\min}}{2}
=\frac{\Omega_t}{2}.
\]
Thus the first lower bound is tight.
\end{proof}

\subsection{Historical persistence and continuous refinement}
\label{app:history-refinement}

We use the notation of
Theorem~\ref{thm:history-mechanism-full-influence}, with
$\mP_t=\mP_t^{(h_t)}$ and $\Delta_t=M_t-M_t^{(0)}$.
Let $\mathcal F$ denote the free positions.
Each row of $\mP_t^{(\lambda)}$ is a probability vector,
and conditioning coordinates remain fixed throughout the update.

\begin{corollary}[Historical persistence and continuous refinement]
\label{cor:history-mechanism-next-update}
\label{cor:history-prediction-transmission}
\label{cor:history-proposal-transmission}
For a fixed coefficient $\alpha_t\in[0,1]$, define
\[
\mathbf U_t^+(\lambda)
=(1-\alpha_t)\mathbf U_t(\lambda)
+\alpha_t\mP_t^{(\lambda)}
\]
on free coordinates, and let
\[
d_t^+(\lambda)
=(\mathbf U_t^+(\lambda))_{j,v}
-(\mathbf U_t^+(\lambda))_{j,x_j}.
\]
Then
\begin{equation}
\mathbf U_t^+(h_t)-\mathbf U_t^+(0)
=(1-\alpha_t)h_t\mathbf D^{jv}
+\alpha_t(\mP_t-\mP_t^{(0)})
\label{eq:full-history-distinct-base-transmission}
\end{equation}
on free coordinates. Consequently,
\begin{equation}
d_t^+(h_t)-d_t^+(0)
=\underbrace{(1-\alpha_t)h_t}_{\text{direct persistence}}
+\underbrace{\alpha_t\Delta_t}_{\text{predictor response}}.
\label{eq:history-mechanism-update}
\end{equation}

Moreover, suppose the minimum winner gap is
\[
\gamma_t
:=\min_{\substack{i\in\mathcal F\\ b\neq x_i}}
\bigl((\mZ_t)_{i,x_i}-(\mZ_t)_{i,b}\bigr)>0
\]
and $0<\alpha_t<\gamma_t/(1+\gamma_t)$. Then
\begin{equation}
\Pi(\mathbf U_t^+(\lambda))=\mathbf x
\qquad (0\leq\lambda\leq h_t).
\label{eq:history-next-hard-invariance}
\end{equation}
Under the same winner-gap and step-size conditions, the natural
update for
$d_t=(\mZ_t)_{j,v}-(\mZ_t)_{j,x_j}<0$ satisfies
\begin{equation}
\begin{gathered}
d_t^+(h_t)-d_t=\alpha_t(M_t-d_t),\\
M_t>d_t
\quad\Longleftrightarrow\quad
d_t<d_t^+(h_t)<0.
\end{gathered}
\label{eq:history-temporal-refinement}
\end{equation}
\end{corollary}

\begin{proof}
By the intervention definition,
\[
\mathbf U_t(h_t)=\mZ_t,
\qquad
\mathbf U_t(0)=\mZ_t-h_t\mathbf D^{jv}.
\]
Subtracting the two Euler updates on free coordinates gives
\[
\begin{aligned}
\mathbf U_t^+(h_t)-\mathbf U_t^+(0)
&=(1-\alpha_t)
  \bigl(\mathbf U_t(h_t)-\mathbf U_t(0)\bigr)
  +\alpha_t(\mP_t-\mP_t^{(0)})\\
&=(1-\alpha_t)h_t\mathbf D^{jv}
  +\alpha_t(\mP_t-\mP_t^{(0)}),
\end{aligned}
\]
proving Equation~\ref{eq:full-history-distinct-base-transmission}.
Since $v\neq x_j$, the matrix $\mathbf D^{jv}$ has entry $1$
at $(j,v)$ and $0$ at $(j,x_j)$. Taking the difference between
these two coordinates yields
\[
\begin{aligned}
d_t^+(h_t)-d_t^+(0)
&=(1-\alpha_t)h_t
  +\alpha_t\bigl(M_t-M_t^{(0)}\bigr)\\
&=(1-\alpha_t)h_t+\alpha_t\Delta_t,
\end{aligned}
\]
which proves Equation~\ref{eq:history-mechanism-update}.

Next, the intervention only decreases a nonwinning coordinate.
Thus, for every $i\in\mathcal F$, $b\neq x_i$, and
$\lambda\in[0,h_t]$,
\[
\begin{aligned}
(\mathbf U_t(\lambda))_{i,x_i}
-(\mathbf U_t(\lambda))_{i,b}
&=(\mZ_t)_{i,x_i}-(\mZ_t)_{i,b}
  +(h_t-\lambda)\mathbf 1\{i=j,\ b=v\}\\
&\geq\gamma_t.
\end{aligned}
\]
Because predictor rows are probability vectors,
\[
(\mP_t^{(\lambda)})_{i,x_i}
-(\mP_t^{(\lambda)})_{i,b}\geq-1.
\]
Combining these inequalities gives
\[
\begin{aligned}
(\mathbf U_t^+(\lambda))_{i,x_i}
-(\mathbf U_t^+(\lambda))_{i,b}
&\geq(1-\alpha_t)\gamma_t-\alpha_t\\
&=\gamma_t-\alpha_t(1+\gamma_t)>0.
\end{aligned}
\]
Hence $x_i$ remains the unique winner at every free position.
Conditioning coordinates are unchanged, so
$\Pi(\mathbf U_t^+(\lambda))=\mathbf x$ throughout the path,
establishing Equation~\ref{eq:history-next-hard-invariance}.

Finally, $\lambda=h_t$ corresponds to the original,
unintervened state. Its updated margin is
\[
d_t^+(h_t)=(1-\alpha_t)d_t+\alpha_t M_t,
\]
and therefore
\[
d_t^+(h_t)-d_t=\alpha_t(M_t-d_t).
\]
Under the stated conditions, $\alpha_t>0$ and
$d_t^+(h_t)<0$ by hard-state preservation. Consequently,
\[
M_t>d_t
\quad\Longleftrightarrow\quad
d_t^+(h_t)>d_t
\quad\Longleftrightarrow\quad
d_t<d_t^+(h_t)<0,
\]
proving Equation~\ref{eq:history-temporal-refinement}.
\end{proof}

Equation~\ref{eq:history-mechanism-update} compares the updated
states with and without historical support, separating direct
persistence from the model's prediction response.
Equation~\ref{eq:history-temporal-refinement} instead compares the
original trajectory before and after one update: under the stated
conditions, $M_t>d_t$ precisely means that the candidate's gap to
the current winner shrinks while the hard-token sequence remains
unchanged.

\subsection{Proof of Theorem~\ref{thm:flm-decoder-accessibility}}
\label{app:flm-decoder-proof}

We use the representations defined in the main text:
\[
\begin{aligned}
 A_t&=(\mathbf C_{\mathcal N},(\mZ_t)_{\mathcal N}),
 &B_t&=(\mathbf C_{\mathcal N},\mathsf H((\mZ_t)_{\mathcal N})),\\
 A_t^{\mathrm{soft}}&=(\mathbf C_{\mathcal N},\mathbf W_t),
 &A_t^{\mathrm{conf}}&=(B_t,\max\mathbf W_t),
\end{aligned}
\]
where $\mathsf H=\operatorname{OH}\circ\Pi$,
$\mathbf W_t=\operatorname{softmax}((\mZ_t)_{\mathcal N})$, and all
operations are rowwise. The receptive field $\mathcal N$ contains $J$
as an identifiable center, and $Y=\mathbf 1\{X_J=v_{\mathrm{path}}\}$.
Fix the trained decoders. All risks and expectations are taken over the
same independent test distribution of puzzles, trajectory randomness,
and uniformly sampled output positions. All logarithms are natural, and
BCE is unweighted.

\begin{proof}
\textbf{Step 1: Decompose the decoder risks.}
For any representation $T$, let $p_T=\Pr(Y=1\mid T)$ and let $f(T)$
be a decoder's predicted probability. Conditioning the BCE loss on $T$
gives
\begin{equation}
\begin{aligned}
 R(f;T)
 &=\mathbb E[-Y\log f(T)-(1-Y)\log(1-f(T))]\\
 &=H(Y\mid T)+\mathbb E\operatorname{KL}
   \bigl(\operatorname{Ber}(p_T)\,\|\,
         \operatorname{Ber}(f(T))\bigr).
\end{aligned}
\label{eq:decoder-excess-risk-identity}
\end{equation}
Thus $\varepsilon_t^Z,\varepsilon_t^{\mathrm{hard}}\geq0$.
Since $B_t$ is a deterministic function of $A_t$, the Bayes-risk gap
$\mathcal I_t:=H(Y\mid B_t)-H(Y\mid A_t)$ is nonnegative, and
\begin{equation}
 \Delta_t^{\mathrm{dec}}
 =\mathcal I_t+\varepsilon_t^{\mathrm{hard}}-\varepsilon_t^Z.
\label{eq:decoder-gap-decomposition}
\end{equation}

\medskip
\noindent\textbf{Step 2: Lower-bound the signal lost under hard projection.}
Softmax preserves rowwise argmax, including ties under the fixed rule.
Hence the representations generate nested sigma-algebras:
\[
 \sigma(B_t)\subseteq\sigma(A_t^{\mathrm{conf}})
 \subseteq\sigma(A_t^{\mathrm{soft}})\subseteq\sigma(A_t).
\]
Set
$p_s=\Pr(Y=1\mid A_t^{\mathrm{soft}})$ and
$p_c=\Pr(Y=1\mid A_t^{\mathrm{conf}})$.
The tower property gives
$p_c=\mathbb E[p_s\mid A_t^{\mathrm{conf}}]$, so
\begin{equation}
 \nu_t=\mathbb E[(p_s-p_c)^2].
\label{eq:decoder-conditional-variance-identity}
\end{equation}
Monotonicity of conditional entropy and the conditional Bernoulli KL
identity yield
\begin{align}
 \mathcal I_t
 &\geq H(Y\mid A_t^{\mathrm{conf}})
        -H(Y\mid A_t^{\mathrm{soft}})\nonumber\\
 &=\mathbb E\operatorname{KL}
   \bigl(\operatorname{Ber}(p_s)\,\|\,
         \operatorname{Ber}(p_c)\bigr)\nonumber\\
 &\geq2\mathbb E[(p_s-p_c)^2]=2\nu_t.
\label{eq:decoder-confidence-signal-bound}
\end{align}
The last inequality follows from
$\operatorname{KL}(\operatorname{Ber}(p)\|\operatorname{Ber}(q))
\geq2(p-q)^2$ in natural logarithms. Indeed, for $q\in(0,1)$, the KL
as a function of $p$ has value and first derivative zero at $p=q$, and
second derivative $1/[p(1-p)]\geq4$; boundary cases follow by limits.
Combining~\eqref{eq:decoder-confidence-signal-bound}
with~\eqref{eq:decoder-gap-decomposition} gives
\[
 \Delta_t^{\mathrm{dec}}
 \geq2\nu_t+\varepsilon_t^{\mathrm{hard}}-\varepsilon_t^Z
 \geq2\nu_t-\varepsilon_t^Z.
\]

\medskip
\noindent\textbf{Step 3: Upper-bound the remaining benefit by state error.}
Let $\widehat X_J=[\Pi(\mZ_t)]_J$ and
$\widehat Y=\mathbf 1\{\widehat X_J=v_{\mathrm{path}}\}$.
Since the center $J$ is identifiable, $\widehat Y$ is measurable from $B_t$.
A binary-label error implies a token error. On the event
$\widehat X_J\neq X_J$, write $\mathbf z=(\mZ_t)_{J,:}$ and
$w=\widehat X_J$. Then $z_w\geq z_{X_J}$, and
\begin{equation}
\begin{aligned}
 \|\mathbf z-\mathbf e_{X_J}^{\top}\|_2^2
 &\geq(z_{X_J}-1)^2+z_w^2\\
 &\geq\tfrac12(1-z_{X_J}+z_w)^2
 \geq\tfrac12.
\end{aligned}
\label{eq:decoder-token-error-distance}
\end{equation}
Therefore
\begin{equation}
\begin{aligned}
 q_t:=\Pr(\widehat Y\neq Y)
 &\leq\Pr(\widehat X_J\neq X_J)\\
 &\leq2\mathbb E\|(\mZ_t)_{J,:}
                   -\mathbf e_{X_J}^{\top}\|_2^2
 \leq2b_t^2.
\end{aligned}
\label{eq:decoder-hard-error-bound}
\end{equation}
Let $E_t=\mathbf 1\{\widehat Y\neq Y\}$. Conditional on $B_t$,
$Y$ and $E_t$ determine one another because $\widehat Y$ is known.
Consequently,
\[
 H(Y\mid B_t)=H(E_t\mid B_t)
 \leq H(E_t)=h_{\mathrm{bin}}(q_t).
\]
If $2b_t^2\leq1/2$, use the monotonicity of $h_{\mathrm{bin}}$ on
$[0,1/2]$; otherwise use $H(Y\mid B_t)\leq\log2$.
In both cases,
\begin{equation}
 0\leq H(Y\mid A_t)\leq H(Y\mid B_t)
 \leq\chi_t:=h_{\mathrm{bin}}(\min\{1/2,2b_t^2\}).
\label{eq:decoder-state-error-entropy-bound}
\end{equation}
Using~\eqref{eq:decoder-gap-decomposition} and nonnegativity of
$H(Y\mid A_t)$ and $\varepsilon_t^Z$, we conclude that
\[
 \Delta_t^{\mathrm{dec}}
 \leq H(Y\mid B_t)+\varepsilon_t^{\mathrm{hard}}
 \leq\chi_t+\varepsilon_t^{\mathrm{hard}}.
\]
Finally,
\[
 R_t^Z=H(Y\mid A_t)+\varepsilon_t^Z
 \leq\chi_t+\varepsilon_t^Z.
\]
Together with Step 2, these are all the asserted bounds.
\end{proof}

The proof does not assume that either trained decoder attains its Bayes
risk. Matching architecture and training budget therefore does not remove
the excess-risk terms.

\section{Experiments}
\label{app:experiments}

% Requires \usepackage{booktabs} and \usepackage{url}.
\subsection{Dataset Construction and Processing}
\label{app:dataset-details}

\subsubsection{Maze Datasets}
\label{app:maze-datasets}

Each of Maze10, Maze15, and Maze20 contains 10,000 training examples, and 1,000 test examples.
The corresponding logical lattices are $10\times10$, $15\times15$,
and $20\times20$, respectively.
We follow the maze-generation setup of VFScale~\citep{zhang2026vfscale}.

For each size, we generate 12,000 candidate mazes with seed 42 and
default generator parameters. Each maze is a tree spanning all lattice
cells. We sample two distinct nodes uniformly without replacement as
the start and goal, and obtain their unique path with the library's
A* solver. We retain paths with at least five logical nodes, equivalent
to four moves, without an upper length limit.
We deduplicate by binary maze layout, excluding endpoint annotations,
then shuffle with NumPy's \texttt{default\_rng} using seed 123.
The first 10,000 retained examples form the training set, followed by 1,000 for testing; remaining examples are unused.
Each retained layout has one start--goal pair, and no data augmentation
is applied. The fixed splits have no overlapping layouts.

An $n\times n$ logical maze is rasterized into a
$(2n+1)\times(2n+1)$ board. Input symbols \texttt{0}, \texttt{1},
\texttt{2}, and \texttt{3} denote traversable pixels, walls, the start,
and the goal, respectively. The target preserves the board and endpoints
and marks intermediate solution pixels with \texttt{4}.
Both boards are flattened in row-major order and serialized as
\texttt{[input][SEP][output][EOS]} with character-level tokenization.
The resulting length is $2(2n+1)^2+2$ tokens before padding.
Table~\ref{tab:maze-statistics} summarizes the representation sizes and
test-set path lengths. Path length counts logical moves; a path of
length $L$ contains $L+1$ logical nodes.

\begin{table}[t]
\centering
\small
\begin{tabular}{lccrr}
\toprule
Dataset & Lattice & Raster & Tokens & Test path length \\
\midrule
Maze10 & $10\times10$ & $21\times21$ & 884   & $26.44\pm14.78$ \\
Maze15 & $15\times15$ & $31\times31$ & 1,924 & $50.74\pm31.43$ \\
Maze20 & $20\times20$ & $41\times41$ & 3,364 & $79.85\pm52.63$ \\
\bottomrule
\end{tabular}
\caption{Maze representation sizes and test-set path lengths
(mean $\pm$ population standard deviation, in logical moves).
Token counts include \texttt{[SEP]} and \texttt{[EOS]} and exclude
padding and CoT actions.}
\label{tab:maze-statistics}
\end{table}

Ground-truth paths are deterministically converted
to logical actions \texttt{5}, \texttt{6}, \texttt{7}, and \texttt{8}
for up, down, left, and right, respectively, with one action per logical
move. The sequence format is
\texttt{[input][SEP][actions]=[output][EOS]}.

\subsubsection{Sudoku Datasets}
\label{app:sudoku-datasets}

We use Sudoku-Extreme\footnote{\href{https://huggingface.co/datasets/sapientinc/sudoku-extreme}{Hugging Face: Sudoku-Extreme.}},
a deduplicated collection of Sudoku puzzles with difficulty ratings
based on solver backtracking counts.
We take the first 1,000,000 examples for training, the following 1,000
for validation, and the next 1,000 for testing.
No additional data augmentation is applied in our experiments.

Each $9\times9$ puzzle and its solution are flattened in row-major order
into strings of length 81. In our representation, \texttt{0} denotes
an empty cell and \texttt{1}--\texttt{9} denote filled cells.
We serialize each pair as
\texttt{[puzzle][SEP][solution][EOS]}.
The puzzle remains clean conditioning, and only the solution span is
generated. With character-level tokenization, each pair contains
164 tokens before padding.
A correct completion must preserve the given clues and contain digits
1 through 9 exactly once in every row, column, and $3\times3$ sub-grid.
Our primary metric is sequence accuracy, defined as exact agreement
with the reference solution over all 81 cells. We also report per-cell
accuracy over the full solution span.

To assess generalization across datasets, we directly evaluate models
trained on Sudoku-Extreme on Standard Sudoku from the Kaggle
collection\footnote{\href{https://www.kaggle.com/datasets/radcliffe/3-million-sudoku-puzzles-with-ratings}{Kaggle: 3 Million Sudoku Puzzles with Ratings.}},
without additional fine-tuning.

\subsection{Implementation Details}
\label{app:Implementation-details}

\subsubsection{Architecture and Optimization Hyperparameters}
\label{app:architecture-optimization}

FLM and the discrete diffusion baselines use GPT-2-style Transformers
with bidirectional self-attention, pre-layer normalization, learned
absolute positional embeddings, and GELU feed-forward layers.
The feed-forward dimension is four times the hidden dimension.
Time conditioning uses a 512-dimensional sinusoidal embedding
followed by a three-layer MLP, whose output is added to the input
representations. At each model size, FLM and the corresponding
discrete baselines share the same backbone configuration.

Both Maze and Sudoku-Extreme use character-level tokenization with
a vocabulary of 31 tokens, including special tokens. Each example
is serialized as \texttt{[input][SEP][output][EOS]}.
The input and separator remain fixed during training and sampling.
The loss is computed only over output positions, including
\texttt{[EOS]}, while padding positions are excluded.
The unpadded sequence lengths are 884, 1,924, and 3,364 for
Maze10, Maze15, and Maze20, respectively, and 164 for Sudoku-Extreme.

\begin{table*}[t]
\centering
\small
\caption{Backbone configurations for Maze and Sudoku-Extreme.
FFN denotes the feed-forward dimension, and maximum length denotes
the capacity of the learned positional embedding table.
Parameter counts include the input and output projections,
positional embeddings, and time-conditioning network.
FLM and the discrete baselines have identical parameter counts
when using the same configuration.}
\label{tab:implementation-architectures}
\begin{tabular}{llrrrrr}
\toprule
Dataset & Model & Hidden dim. & Heads & FFN dim. & Max. length & Parameters \\
\midrule
Maze10 & 6 layers  & 768  & 6  & 3072 & 1024 & 44.94M \\
       & 8 layers  & 768  & 8  & 3072 & 1024 & 59.11M \\
       & 12 layers & 768  & 12 & 3072 & 1024 & 87.47M \\
\midrule
Maze15 & 8 layers  & 768  & 8  & 3072 & 2048 & 59.90M \\
       & 12 layers & 768  & 12 & 3072 & 2048 & 88.25M \\
       & 14 layers & 896  & 14 & 3584 & 2048 & 139.00M \\
\midrule
Maze20 & 14 layers & 896  & 14 & 3584 & 4096 & 140.83M \\
       & 20 layers & 1024 & 16 & 4096 & 4096 & 258.81M \\
       & 24 layers & 1024 & 16 & 4096 & 4096 & 309.19M \\
\midrule
Sudoku-Extreme & 12 layers & 768 & 12 & 3072 & 256 & 86.88M \\
\bottomrule
\end{tabular}
\end{table*}

We optimize all models with AdamW~\citep{loshchilov2017fixing},
using $\beta_1=0.9$, $\beta_2=0.95$, weight decay 0.01,
and gradient clipping at a global norm of 1.0.
The training budgets are 1,000 epochs for FLM and 2,000 epochs
for the discrete baselines on each Maze dataset, and 300 epochs
for all methods on Sudoku-Extreme.
Table~\ref{tab:implementation-optimization} summarizes the
optimization settings. Reported batch sizes are global,
accounting for all devices and gradient accumulation.
All evaluations use exponential moving average (EMA) weights
with decay 0.9999.

\begin{table*}[t]
\centering
\small
\caption{Optimization settings for Maze and Sudoku-Extreme.}
\label{tab:implementation-optimization}
\begin{tabular}{lcc}
\toprule
Hyperparameter & Maze & Sudoku-Extreme \\
\midrule
Optimizer & AdamW & AdamW \\
AdamW $(\beta_1,\beta_2)$ & $(0.9,0.95)$ & $(0.9,0.95)$ \\
Peak learning rate & $4\times10^{-4}$ & $1\times10^{-4}$ \\
Minimum learning rate & $4\times10^{-6}$ & $1\times10^{-6}$ \\
Learning-rate schedule & Cosine decay & Cosine decay \\
Linear warmup (optimizer updates) & 125 & 500 \\
Global batch size (default) & 256 & 1024 \\
Weight decay & 0.01 & 0.01 \\
Training dropout & 0.1 & 0.1 \\
Gradient clipping & 1.0 & 1.0 \\
EMA decay & 0.9999 & 0.9999 \\
FLM training budget (epochs) & 1000 & 300 \\
Discrete baseline training budget (epochs) & 2000 & 300 \\
\bottomrule
\end{tabular}
\end{table*}

FLM is trained on linear interpolations between independent
standard Gaussian noise and one-hot target tokens, with the
input conditions held fixed. Interpolation times are sampled
uniformly on $[0,1]$ using antithetic pairs.
We minimize clean-token cross-entropy over the supervised
output positions. The model predicts a clean-token distribution
at every time step, from which the continuous velocity field
is constructed.
D3PM, SEDD, and MDLM retain their respective discrete diffusion
objectives: a variational-bound objective, score entropy,
and masked denoising, respectively. Their corruption processes
use an absorbing mask token and preserve the same input conditions.

\subsubsection{Inference Details}
\label{app:inference-details}

We consider sampling budgets
$K\in\{1,2,4,8,16,32,64,128,256,512\}$.
For FLM, $K$ denotes the number of integration steps on the
uniform grid $t_k=k/K$. Sampling starts from standard Gaussian
noise at output positions. Given the predicted clean-token
probabilities $p_\theta(z_t,t\mid c)$, where $c$ denotes the
fixed input conditions, intermediate Euler updates take the form
\begin{equation}
z_{t_{k+1}}
=
z_{t_k}
+
\frac{1}{K}
\frac{p_\theta(z_{t_k},t_k\mid c)-z_{t_k}}
     {1-t_k+10^{-5}}.
\end{equation}
For the final interval, we directly set the state to the
predicted clean-token distribution.
Conditioning positions are restored after every update,
and tokenwise argmax is applied to the final continuous state.
This procedure requires exactly $K$ model evaluations.

For the discrete baselines, output positions are initialized
with the absorbing mask token. We use a uniform reverse-time
grid from 1 to a terminal cutoff $\epsilon$, with ancestral
reverse updates for D3PM and MDLM and score-based reverse
updates for SEDD. Categorical sampling uses temperature 1.0.

After the $K$ reverse updates, the discrete implementations
perform one additional noise-removal model evaluation.
Thus, $K$ counts reverse-process updates for these methods,
and the corresponding number of model evaluations is $K+1$.
Input conditions remain fixed throughout sampling,
and all methods use their EMA weights.

\subsubsection{FLM decoder
experiments' implementation details}
\label{app:flm-decoder-experiments}
\paragraph{Experimental setup.}
We freeze a Maze15 FLM with eight layers, eight attention heads, and
hidden width 768, using EMA weights, and cache states $Z_0,\ldots,Z_8$
from one $N=8$ trajectory per maze. Separate local CNN decoders predict
the binary solution-path mask from the maze condition and each state
representation at each step. All decoders share a $9\times9$ receptive
field, 38,177 parameters, and a 480-update training budget. We use
512/128/128 disjoint training/validation/test mazes and three decoder
initialization seeds, matching initialization and minibatch order across
conditions. Training uses class-weighted BCE; evaluation reports
unweighted BCE and path-F1 at a fixed probability threshold of 0.5.

\subsection{Full-state interventions on retained historical support}
\label{app:full-state-intervention}

\paragraph{Experimental setup.}
We test whether retained historical support affects the next prediction
while the current hard-token sequence remains unchanged.
For a strict non-top-1 candidate $v$ at position $j$, we intervene along
\[
\mathbf U_t(\lambda)
=\mZ_t-(h_t-\lambda)\mathbf e_j\mathbf e_v^\top,
\qquad 0\leq\lambda\leq h_t.
\]
This removes only the historical contribution at $(j,v)$, leaving
context, time, other coordinates, and initial noise fixed.
We recompute the complete model without renormalizing the state.
The measured effect is $\Delta_t=M_t^{(h_t)}-M_t^{(0)}$, where
$M_t^{(\lambda)}$ is the predicted probability of $v$ minus that of
the original winner $x_j$. Thus, $\Delta_t>0$ means that retaining
history increases the candidate's relative support.

The original cohort contains 898 fixed events from Sudoku and Maze20
trajectories with $N\in\{2,4\}$, including 261 correct-candidate events
($v$ equals the target token). We evaluate 17 points per path.
This is a retrospective mechanism analysis.
Reported means first average over noise seeds within each problem
and then weight problems equally; nominal 95\% confidence intervals
use 20,000 problem-cluster bootstrap resamples without
multiple-comparison correction.

\paragraph{Historical support changes predictions.}
Retaining history increases the correct-candidate margin in 232/261
original events, with positive means and confidence intervals excluding
zero in all four groups (Table~\ref{tab:full-history-endpoints}).
An additional validation cohort contains 181 eligible events, including
48 correct-candidate events, of which 47 have positive endpoint effects.
These inputs differ from earlier mechanism inputs, but are not claimed
to be unseen during checkpoint development; some initial-noise keys
are shared with the original cohort. The additional Maze20 $N=4$
group has only five correct-candidate events, so its result remains
limited by sample size.
For any nonzero $\Delta_t$,
Equation~\ref{eq:history-mechanism-hard-error} implies a worst-case
margin-reproduction error of at least $|\Delta_t|/2$ for a readout
using only the fixed context, time, and hard state.

\begin{table}[t]
\centering
\small
\setlength{\tabcolsep}{6pt}
\caption{Effects of retaining history on correct-candidate probability
margins, in percentage points. Means and confidence intervals weight
problems equally; the last column counts events with positive effects.}
\label{tab:full-history-endpoints}
\begin{tabular}{@{}lcc@{}}
\toprule
Setting & $\Delta_t$ [95\% CI] & Positive events \\
\midrule
\multicolumn{3}{@{}l}{\textit{Original cohort}} \\
Sudoku, $N=2$ & $1.456\ [0.767,2.293]$ & 60/73 \\
Sudoku, $N=4$ & $1.634\ [0.961,2.406]$ & 63/71 \\
Maze20, $N=2$ & $0.188\ [0.045,0.356]$ & 66/73 \\
Maze20, $N=4$ & $0.921\ [0.316,1.641]$ & 43/44 \\
\midrule
\multicolumn{3}{@{}l}{\textit{Additional validation cohort}} \\
Sudoku, $N=2$ & $1.168\ [0.213,2.361]$ & 18/19 \\
Sudoku, $N=4$ & $1.439\ [0.231,3.084]$ & 10/10 \\
Maze20, $N=2$ & $0.137\ [0.000117,0.389]$ & 14/14 \\
Maze20, $N=4$ & $0.345\ [0.017,0.964]$ & 5/5 \\
\bottomrule
\end{tabular}
\end{table}

\paragraph{How the response propagates.}
Following Theorem~\ref{thm:history-mechanism-full-influence}, we
integrate the four signed contributions along the same full-state path:
\[
I_X=\int_0^{h_t}\mathcal C_X(\lambda)\,d\lambda,
\qquad X\in\{Q,K,V,R\}.
\]
The calculation includes all first-layer heads and receiving positions,
with gradients through the entire downstream network.
Q has the largest mean contribution in each group, and V is positive
in all four groups (Table~\ref{tab:full-history-contributions}).
The bypass contribution R is positive in Sudoku and near zero in
Maze20; it is distinct from the solver's direct state-persistence term.
K intervals include zero, which does not establish that K is unused.

V-other, the contribution through receiving positions $i\neq j$,
is positive with intervals excluding zero in every group.
This shows that history at $j$ can affect its final prediction through
values read by other positions and processed by subsequent layers,
as described by Equation~\ref{eq:history-cross-position}.
In Sudoku, V-other exceeds total V, reflecting cancellation with the
contribution at receiver $j$; these are signed effects, not mechanism
percentages.

\begin{table}[t]
\centering
\small
\setlength{\tabcolsep}{4pt}
\caption{Signed path-integral contributions for correct candidates in
 the original cohort, in probability percentage points. Each cell gives
 the mean and nominal 95\% confidence interval. V-other is the part of V
 transmitted through first-layer receiving positions $i\neq j$.
 Negative estimates rounded to zero are shown as $-0.000$.}
\label{tab:full-history-contributions}
\begin{tabular}{@{}lcccc@{}}
\toprule
Pathway & \shortstack{Sudoku\\$N=2$}
& \shortstack{Sudoku\\$N=4$}
& \shortstack{Maze20\\$N=2$}
& \shortstack{Maze20\\$N=4$} \\
\midrule
Q
& \shortstack{$1.001$\\$[0.527,1.576]$}
& \shortstack{$1.114$\\$[0.655,1.641]$}
& \shortstack{$0.148$\\$[0.040,0.276]$}
& \shortstack{$0.656$\\$[0.235,1.151]$} \\
\addlinespace
K
& \shortstack{$-0.002$\\$[-0.005,0.001]$}
& \shortstack{$0.000$\\$[-0.002,0.002]$}
& \shortstack{$-0.000$\\$[-0.002,0.001]$}
& \shortstack{$-0.005$\\$[-0.012,0.001]$} \\
\addlinespace
V
& \shortstack{$0.052$\\$[0.019,0.096]$}
& \shortstack{$0.050$\\$[0.017,0.085]$}
& \shortstack{$0.040$\\$[0.007,0.082]$}
& \shortstack{$0.270$\\$[0.070,0.522]$} \\
\addlinespace
R
& \shortstack{$0.405$\\$[0.214,0.635]$}
& \shortstack{$0.470$\\$[0.274,0.696]$}
& \shortstack{$-0.001$\\$[-0.001,0.000]$}
& \shortstack{$-0.000$\\$[-0.002,0.001]$} \\
\midrule
V-other
& \shortstack{$0.093$\\$[0.043,0.160]$}
& \shortstack{$0.099$\\$[0.053,0.153]$}
& \shortstack{$0.032$\\$[0.006,0.064]$}
& \shortstack{$0.223$\\$[0.049,0.450]$} \\
\bottomrule
\end{tabular}
\end{table}

These results support computational use of retained historical support,
with positive effects for most evaluated correct candidates.
They concern the next prediction margin, not final puzzle accuracy,
and do not establish uniformly beneficial effects or positive
derivatives throughout a continuous path.

% Requires \usepackage{booktabs}
\subsection{Complete Results across Model Capacities}
\label{app:maze-capacity}

Table~\ref{tab:maze-capacity-full} reports the complete model-capacity
sweep underlying Figures~\ref{fig:maze-capacity-curves}
and~\ref{fig:maze-capacity-fixed-steps}.
For each maze task, we evaluate FLM and MDLM at three matched model
sizes across ten denoising step budgets.

\begin{table*}[t]
    \centering
    \caption{
    Sequence accuracy (\%) of FLM and MDLM across model capacities
    and denoising steps on Maze10, Maze15, and Maze20.
    Parameter counts are matched between methods at each model size.
    }
    \label{tab:maze-capacity-full}
    \setlength{\tabcolsep}{3pt}

    \begin{tabular*}{\textwidth}{
        @{\extracolsep{\fill}}ll*{10}{c}@{}
    }
        \toprule
        Parameters & Method
        & \multicolumn{10}{c}{Denoising steps} \\
        \cmidrule(lr){3-12}
        & & 1 & 2 & 4 & 8 & 16 & 32 & 64 & 128 & 256 & 512 \\
        \midrule

        \multicolumn{12}{l}{\textit{Maze10}} \\
        \midrule
        44.94M & FLM
        & 57.6 & 72.7 & 76.0 & 77.2 & 77.7
        & 78.0 & 78.3 & 78.3 & 78.3 & 78.3 \\
        & MDLM
        & 51.1 & 57.4 & 64.4 & 69.0 & 78.1
        & 84.3 & 87.3 & 89.9 & 91.0 & 90.8 \\
        \addlinespace
        59.11M & FLM
        & 95.7 & 99.5 & 99.5 & 99.5 & 99.5
        & 99.5 & 99.6 & 99.6 & 99.6 & 99.6 \\
        & MDLM
        & 87.2 & 88.8 & 92.1 & 94.6 & 95.2
        & 98.3 & 98.5 & 98.6 & 98.2 & 98.6 \\
        \addlinespace
        87.47M & FLM
        & 100.0 & 100.0 & 100.0 & 100.0 & 100.0
        & 100.0 & 100.0 & 100.0 & 100.0 & 100.0 \\
        & MDLM
        & 97.2 & 97.6 & 98.5 & 99.0 & 99.2
        & 99.7 & 99.9 & 99.8 & 99.9 & 99.7 \\
        \midrule

        \multicolumn{12}{l}{\textit{Maze15}} \\
        \midrule
        59.90M & FLM
        & 58.3 & 78.7 & 82.9 & 84.0 & 84.3
        & 84.4 & 84.4 & 84.4 & 84.4 & 84.4 \\
        & MDLM
        & 27.9 & 31.2 & 39.2 & 45.0 & 54.3
        & 64.0 & 73.5 & 79.1 & 82.0 & 82.1 \\
        \addlinespace
        88.25M & FLM
        & 95.9 & 98.9 & 99.1 & 99.1 & 99.1
        & 99.1 & 99.2 & 99.2 & 99.2 & 99.2 \\
        & MDLM
        & 49.8 & 54.9 & 62.5 & 69.6 & 76.2
        & 83.7 & 88.3 & 91.8 & 93.7 & 95.3 \\
        \addlinespace
        139.00M & FLM
        & 99.1 & 99.8 & 99.9 & 99.9 & 99.9
        & 99.9 & 99.9 & 99.9 & 99.9 & 99.9 \\
        & MDLM
        & 74.7 & 79.2 & 83.1 & 87.3 & 91.5
        & 94.7 & 96.2 & 97.7 & 96.9 & 97.1 \\
        \midrule

        \multicolumn{12}{l}{\textit{Maze20}} \\
        \midrule
        140.83M & FLM
        & 33.2 & 47.0 & 51.2 & 53.1 & 53.6
        & 53.7 & 53.9 & 54.1 & 54.2 & 55.7$^{*}$ \\
        & MDLM
        & 26.2 & 29.0 & 33.9 & 38.1 & 45.2
        & 52.6 & 63.4 & 72.8 & 77.0 & 79.2 \\
        \addlinespace
        258.81M & FLM
        & 71.4 & 88.1 & 91.3 & 92.1 & 92.4
        & 92.5 & 92.5 & 92.6 & 92.6 & 92.9$^{*}$ \\
        & MDLM
        & 35.6 & 41.2 & 48.2 & 55.8 & 63.0
        & 70.6 & 80.2 & 83.7 & 87.9 & 90.1 \\
        \addlinespace
        309.19M & FLM
        & 100.0 & 100.0 & 100.0 & 100.0 & 100.0
        & 100.0 & 100.0 & 100.0 & 100.0 & 100.0 \\
        & MDLM
        & 43.0 & 50.4 & 55.9 & 64.5 & 70.7
        & 79.6 & 84.9 & 89.5 & 92.5 & 92.9 \\
        \bottomrule
    \end{tabular*}
\end{table*}

\subsection{Wall-Clock Inference Efficiency}
\label{app:inference-efficiency}

We measure inference latency on a single A100 80\,GB GPU using FP32 with TF32 disabled. Methods use matched Transformer depths and widths within each task. We report mean latency over 32 shared test instances at batch size one, following two warm-up calls. Timing uses GPU synchronization and covers the complete sampler, excluding model loading, tokenization, data transfer, and scoring.

Here, $N$ denotes nominal sampling steps. FLM uses $N$ model evaluations, while standard discrete samplers use $N+1$, including terminal denoising. Confidence remasking uses 2 and 128 evaluations at $N=1$ and $N=64$, respectively. AR+CoT uses cached greedy decoding with fixed output lengths of 480, 1024, and 1760 tokens for Maze10, Maze15, and Maze20, with EOS stopping disabled.

\begin{table*}[t]
\centering
\small
\caption{
Mean inference latency in seconds per instance on one A100 (batch size one; lower is better). $\dagger$ marks randomly initialized models used solely for architecture-level timing; other entries use trained EMA weights. Sampling budgets do not enforce matched accuracy.
}
\label{tab:maze-inference-latency}
\begin{tabular}{lrrrrrr}
\toprule
& \multicolumn{2}{c}{Maze10}
& \multicolumn{2}{c}{Maze15}
& \multicolumn{2}{c}{Maze20} \\
\cmidrule(lr){2-3}
\cmidrule(lr){4-5}
\cmidrule(lr){6-7}
Method
& $N=1$ & $N=64$
& $N=1$ & $N=64$
& $N=1$ & $N=64$ \\
\midrule
FLM
& \textbf{0.013} & 0.656
& \textbf{0.033} & 2.079
& \textbf{0.151} & 9.740 \\
MDLM
& 0.023 & 0.693 & 0.067 & 2.144 & 0.303 & 9.845 \\
MDLM + Entropy
& 0.022 & 0.714 & 0.067 & 2.162 & 0.303 & 9.902 \\
MDLM + Confidence
& 0.023 & 1.371 & 0.066 & 4.204 & 0.303 & 19.392 \\
SEDD
& $0.023^\dagger$ & $0.705^\dagger$
& 0.067 & 2.152 & 0.305 & 9.950 \\
SEDD + Entropy
& $0.026^\dagger$ & $0.724^\dagger$
& 0.067 & 2.172 & 0.307 & 9.978 \\
SEDD + Confidence
& $0.022^\dagger$ & $1.369^\dagger$
& 0.067 & 4.225 & 0.306 & 19.587 \\
D3PM
& $0.025^\dagger$ & $0.645^\dagger$
& $0.065^\dagger$ & $2.082^\dagger$
& 0.301 & 9.808 \\
ReMDM
& 0.021 & 0.651 & 0.066 & 2.089 & 0.301 & 9.772 \\
\midrule
AR+CoT
& \multicolumn{2}{c}{1.424}
& \multicolumn{2}{c}{4.328}
& \multicolumn{2}{c}{$17.384^\dagger$} \\
\bottomrule
\end{tabular}
\end{table*}

FLM achieves the lowest latency at $N=1$ on all three tasks.
At $N=64$, its cost is comparable to standard discrete
samplers, while confidence remasking approximately doubles
latency. Together with Table~\ref{tab:maze-main}, these results
show that FLM combines strong low-step accuracy with low
wall-clock inference cost.

\end{document}